\documentclass{article} 
\usepackage{nips13submit_e,times}
\nipsfinalcopy
\usepackage[hyphens]{url}
\usepackage{url}
\usepackage{times}
\usepackage{graphicx} 
\usepackage{subfigure} 
\usepackage{natbib}
\usepackage{microtype}

\usepackage{algorithm}
\usepackage{algorithmic}
\usepackage{bm}
\usepackage{amsmath}
\usepackage{amssymb}

\newcommand{\reals}{\mathbb{R}}
\newcommand{\ignore}[1]{} 
\usepackage{algorithm}
\usepackage{algorithmic}
\usepackage{amsmath}
\usepackage{amsthm}
\newtheorem{theorem}{Theorem}
\newtheorem{lemma}{Lemma}
\newtheorem{corollary}{Corollary}
\newtheorem{definition}{Definition}
\usepackage{ dsfont }
\newcommand{\expm}{{\textrm{Expm}}}
\newcommand{\Od}{\mathcal{O}_d}
\newcommand{\On}{\mathcal{O}_n}
\newcommand{\Om}{\mathcal{O}_m}

\renewcommand{\algorithmicrequire}{\textbf{Input:}}
\renewcommand{\algorithmicensure}{\textbf{Output:}}

\usepackage{atbeginend}
\usepackage{fancyhdr}

\renewcommand{\headrulewidth}{0pt}
\renewcommand{\footrulewidth}{0pt}
\AfterBegin{itemize}{\addtolength{\itemsep}{-0.1 in}}

\title{Efficient coordinate-descent for orthogonal matrices through Givens rotations}
\author{
Uri Shalit  \\
ICNC-ELSC \& Computer Science Department\\
Hebrew University of Jerusalem\\
91904 Jerusalem, Israel \\
\texttt{uri.shalit@mail.huji.ac.il} \\
\And
Gal Chechik \\
The Gonda Brain Research Center \\
Bar Ilan University \\
52900 Ramat-Gan, Israel \\
\texttt{gal.chechik@biu.ac.il} \\
}

\begin{document}
\thispagestyle{fancy}%
\vskip -3pt
\maketitle

\begin{abstract}

Optimizing over the set of orthogonal matrices is a central component
in problems like sparse-PCA or tensor decomposition. Unfortunately, such
optimization is hard since simple operations on orthogonal matrices
easily break orthogonality, and correcting orthogonality usually
costs a large amount of computation.
Here we propose a framework for optimizing orthogonal matrices, that
is the parallel of coordinate-descent in Euclidean spaces. It is based
on {\em Givens-rotations}, a fast-to-compute operation that affects a
small number of entries in the learned matrix, and preserves
orthogonality.
We show two applications of this approach: an algorithm for tensor
decomposition that is used in learning mixture models, and an
algorithm for sparse-PCA. We study the parameter regime where a 
Givens rotation approach converges faster and achieves a superior
model on a genome-wide brain-wide mRNA expression dataset.

\end{abstract}
\vspace{-5pt}
\section{Introduction}
Optimization over orthogonal matrices -- matrices whose rows and
columns form an orthonormal basis of $\reals^d$ -- is central to many
machine learning optimization problems. Prominent examples include
{\em Principal Component Analysis} (PCA), {\em Sparse PCA}, and {\em
Independent Component Analysis (ICA)}. In addition, many new
applications of tensor orthogonal decompositions were introduced
recently, including Gaussian Mixture Models, Multi-view Models and
Latent Dirichlet Allocation (e.g., \citet{anandkumar2012tensor,hsu2013learning}).

A major challenge when optimizing over the set of orthogonal matrices
is that simple updates such as matrix addition usually break
orthonormality. Correcting by orthonormalizing a matrix $V \in
\reals^{d \times d}$ is typically a costly procedure: even a change to
a single element of the matrix, may require $O(d^3)$ operations in the
general case for re-orthogonalization.

In this paper, we present a new approach for optimization over the
manifold of orthogonal matrices, that is based on a series of sparse
and efficient-to-compute updates that operate {\bf within the set of
orthonormal matrices}, thus saving the need for costly
orthonormalization. The approach can be seen as the equivalent of
coordinate descent in the manifold of orthonormal matrices. Coordinate
descent methods are particularly relevant for problems that are too big
to fit in memory, for problems where one might be satisfied with a
partial answer, or in problems where not all the data is available at
one time \citep{richtarik2012iteration}.

We start by showing that the orthogonal-matrix equivalent of a single
coordinate update is applying a single {\em{Givens rotation}} to the
matrix.  In section \ref{sec:conv} we prove that for a
differentiable objective the procedure converges to a local
optimum under minimal conditions, and prove an $O(1/T)$ convergence rate for the norm of the
gradient.  Sections \ref{sec:spca} and \ref{sec:tensor} describe two
applications: (1) sparse PCA, including a variant for streaming data;
(2) a new method for orthogonal tensor decomposition. We study how the
performance of the method depends on the problems hyperparameters
using synthetic data, and demonstrate that it achieves superior
accuracy on an application of sparse-PCA for analyzing gene expression
data.

\section{Coordinate descent on the orthogonal matrix manifold}

Coordinate descent (CD) is an efficient alternative to gradient descent
when the cost of computing and applying a gradient step at a single
coordinate is small relative to computing the full gradient. In these
cases, convergence can be achieved with a smaller number of computing
operations, although using a larger number of (faster) steps.

Applying coordinate descent to optimize a function involves choosing a
coordinate basis, usually the standard basis. Then calculating a directional
derivative in the direction of one of the coordinates. And finally, updating
the iterate in the direction of the chosen coordinate.

%

To generalize CD to operate over the set of orthogonal matrices, we
need to generalize these ideas of directional derivatives and updating
the orthogonal matrix in a ``straight direction''.

In the remaining of this section, we introduce the set of orthogonal
matrices, $\Od$, as a Riemannian manifold. We then show that applying
coordinate descent to the Riemannian gradient amounts to multiplying
by Givens rotations.  Throughout this section and the next, the
objective function is assumed to be a differentiable function $f: \Od
\rightarrow \reals$.

\subsection{The orthogonal manifold and Riemannian gradient}
The orthogonal matrix manifold $\Od$ is the set of $d \times d$
matrices $U$ such that $U U^T = U^T U = I_d$. It is a
 dimensional smooth manifold, and is an embedded
submanifold of the Euclidean space $R^{d \times d}$
\citep{absil2009optimization}.

Each point $U \in \Od$ has a tangent space associated with it,  a
$\frac{d(d-1)}{2}$ dimensional vector space, that we will use
below in order to capture the notion of 'direction' on the manifold.
The tangent space is denoted $T_U \Od$, and defined by $ T_U \Od = \{
Z \in \reals^{d\times d}, Z = U \Omega : \Omega = -\Omega^T \} = U
Skew(d), $ where $Skew(d)$ is the set of skew-symmetric $d
\times d$ matrices.

\subsubsection{Geodesic directions}
The natural generalization of straight lines to manifolds are {\em
geodesic curves}. A geodesic curve is locally the ``shortest'' curve between
two points on the manifold, or equivalently, a curve with no
acceleration tangent to the manifold
\citep{absil2009optimization}. For a point $U \in \Od$ and a
``direction'' $U\Omega \in T_U \Od$ there exists a single geodesic
line that passes through $U$ in direction $\Omega$.  Fortunately,
while computing a geodesic curve in the general case might be hard,
computing it for the orthogonal matrix manifold has a closed form
expression: $\gamma:(-1,1) \rightarrow \Od$, $\gamma(\theta) =
U\textrm{Expm} (\theta \Omega)$, where $\gamma(\theta)$ with
$\theta\in(-1,1)$ is the parameterization of the curve, and $\expm$ is
the matrix exponential function.

In the special case where the operator $Expm(\Omega)$ is applied to a
skew-symmetric matrix $\Omega$, it maps $\Omega$ into an orthogonal
matrix \footnote{Because $\expm(\Omega) \expm(\Omega)^T =
\expm(\Omega) \expm(\Omega^T) = \expm(\Omega) \expm(-\Omega) = I$}. As
a result, $\gamma(\theta) = U\textrm{Expm} (\theta \Omega)$ is also an
orthogonal matrix for all $-1<\theta<1$. This provides a useful
parametrization for orthogonal matrices.

\subsubsection{The directional derivative}
In analogy to the Euclidean case, the Riemannian directional
derivative of $f$ in the direction of a vector $U\Omega \in T_U \Od
$ is defined as the derivative of a single variable function which
involves looking at $f$ along a single curve
\citep{absil2009optimization}:

\begin{equation}
\label{eq:deriv}
\nabla_{\Omega} f(U) \equiv \frac{\textrm{d}}{\textrm{d}\theta} f(\gamma(\theta))\Big{|}_{\theta=0} = \frac{\textrm{d}}{\textrm{d}\theta} f(U\textrm{Expm} (\theta \Omega))\Big{|}_{\theta=0}.
\end{equation}
Note that $\nabla_{\Omega} f(U)$ is a scalar.
The definition means that the directional derivative is the limit of $f$ along the geodesic curve going through $U$ in the direction $ U\Omega$.

\subsubsection{The directional update}
Since the Riemannian equivalent of walking in a straight line is
walking along the geodesic curve, taking a step of size $\eta > 0$
from a point $U \in \Od$ in direction $U \Omega \in T_U \Od$ amounts
to:
\begin{equation}
\label{eq:update}
U_{next} = U \expm\left(\eta \Omega\right),
\end{equation}

We also have to define the orthogonal basis for $Skew(d)$. Here we
use $\{e_i e_j^T - e_j e_i^T : 1\leq i < j \leq d\}$.  We denote each
basis vector as $H_{ij} = e_i e_j^T - e_j e_i^T$, $1\leq i < j \leq
d$.

\subsection{Givens rotations as coordinate descent}
Coordinate descent is a popular method of optimization in Euclidean
spaces. It can be more efficient than computing full gradient steps
when it is possible to (1) compute efficiently the coordinate
directional derivative, and (2) apply the update efficiently. We will
now show that in the case of the orthogonal manifold, applying the
update (step 2) can be achieved efficiently. The cost of computing the
coordinate derivative (step 1) depends on the specific nature of the
objective function $f$, and we we show below several cases where that
can be achieved efficiently.

Let $H_{ij}$ be a coordinate direction, let $\nabla_{H_{ij}} f(U)$ be
the corresponding directional derivative, and choose step size $\eta
>0$. \ignore{he choice of $\eta$ can be based on a Lipschitz constant,
on the second derivative, or according to some optimization schedule
such as decreasing size). }  
A straightforward calculation based on Eq. \ref{eq:update} shows that
the update $U_{next} = U \expm (-\eta H_{ij})$ obeys
 \ignore{Recall that $H_{ij} = e_i e_j^T - e_j e_i^T$.} 

$$\expm(-\eta H_{ij}) = \quad \quad\quad\quad\quad $$
\begin{equation*}\label{eq:defgiv}
 \begin{bmatrix}   
                         1   & \cdots &    0   & \cdots &    0   & \cdots &    0   \\
                      \vdots & \ddots & \vdots &        & \vdots &        & \vdots \\
                         0   & \cdots &  cos(\eta)   & \cdots &    -sin(\eta)   & \cdots &    0   \\
                      \vdots &        & \vdots & \ddots & \vdots &        & \vdots \\
                         0   & \cdots &   sin(\eta)   & \cdots &    cos(\eta)   & \cdots &    0   \\
                      \vdots &        & \vdots &        & \vdots & \ddots & \vdots \\
                         0   & \cdots &    0   & \cdots &    0   & \cdots &    1
       \end{bmatrix} 
\end{equation*}    

This matrix is known as a \textit{Givens rotation}
\citep{golub2012matrix} and is denoted $G(i,j,-\eta)$. It has
$cos(\eta)$ at the $(i,i)$ and $(j,j)$ entries, and $\pm sin(\eta)$ at
the $(j,i)$ and $(i,j)$ entries. It is a simple and sparse orthogonal
matrix. For a dense matrix $A \in \reals^{d \times d}$, the linear
operation $A \mapsto AG(i,j,\eta)$ rotates the $i^{th}$ and $j^{th}$
columns of $A$ by an angle $\eta$ in the plane they span. Computing this
operation costs $6d$ multiplications and additions. As a result,
computing Givens rotations successively for all $\frac{d (d-1)}{2}$
coordinates $H_{ij}$ takes $O(d^3)$ operations, the same order as
ordinary matrix multiplication. Therefore the relation between the
cost of a single Givens relative to a full gradient update is the same
as the relation between the cost of a single coordinate update and a
full update is in Euclidean space. We note that any determinant-1 orthogonal matrix
can be decomposed into at most $\frac{d (d-1)}{2}$ Givens rotations.

\subsection{The givens rotation coordinate descent algorithm}

Based on the definition of givens rotation, a natural algorithm for
optimizing over orthogonal matrices is to perform a sequence of
rotations, where each rotation is equivalent to a coordinate-step in
CD.

To fully specify the algorithm we need two more ingredients: (1)
Selecting a schedule for going over the coordinates and (2) Selecting
a step size. For scheduling, we chose here to use a random order of
coordinates, following many recent coordinate descent papers
\citep{richtarik2012iteration,nesterov2012efficiency,patrascu2013efficient}.

For choosing the step size $\eta$ we use exact minimization,
since we found that for the problems we attempted to solve, using
exact minimization was usually the same order of complexity as
performing approximate minimization (like using an Armijo step rule
\citet{bertsekas1999nonlinear, absil2009optimization}). 

\ignore{ Specifically, To find a step size $\eta$ which minimizes
$f(U G(i,j,\eta)) = \hat{f}(\eta)$, we note that the function
$\hat{f}(\eta)$ is a single-variable periodic function. We therefore
minimize over $\eta$ by obtaining a closed form solution or by
standard methods for minimizing a single variable function on a
bounded interval.}

Based on these two decisions, Algorithm (\ref{alg1}) is a random
coordinate minimization technique.
\begin{algorithm}
\caption{Riemannian coordinate minimization on $\Od$}
\label{alg1}
\begin{algorithmic}
\REQUIRE Differentiable objective function $f$, initial matrix $U_0 \in \Od$
\STATE $t=0$
\WHILE{not converged}
\STATE 1. Sample uniformly at random a pair $(i(t),j(t))$ such that $1 \leq i(t) < j(t) \leq d$.
\STATE 2. $\theta_{t+1} = \underset{\theta }{\operatorname{argmin}} \; f\left(U_t \cdot G(i,j,\theta)\right)$.
\STATE 3. $U_{t+1} = U_t \cdot G(i,j,\theta_{t+1})$.
\STATE 4. $t = t+1$.
\ENDWHILE
\ENSURE $U_{final}$.
\end{algorithmic}
\end{algorithm}

\section{Convergence rate for Givens coordinate minimization}
\label{sec:conv}
In this section, we show that under the assumption that the objective
function $f$ is differentiable Algorithm 1 converges to critical point of the
function $f$, and the only stable convergence points are local minima.
We further show that the expectation w.r.t. the random
choice of coordinates of the squared $l_2$-norm of the Riemannian
gradient converges to $0$ with a rate of $O(\frac{1}{T})$ where $T$ is
the number of iterations.  The proofs, including some auxiliary
lemmas, are provided in the supplemental material. Overall we provide
the same convergence guarantees as provided in standard non-convex
optimization (e.g., \citet{nemirovskioptimization,bertsekas1999nonlinear}). 

\begin{definition}{Riemannian gradient}\\
The Riemannian gradient $\nabla f(U)$ of $f$ at point $U \in \Od$ is
the matrix $U \Omega$, where $\Omega \in Skew(d)$, $\Omega_{ji} = -
\Omega_{ij} = \nabla_{ij}f(U), 1 \leq i < j \leq d$ is defined to be the directional
derivative as given in Eq. \ref{eq:deriv}, and $\Omega_{ii} = 0$.
The norm of the Riemannian gradient $||\nabla f(U) ||^2 = Tr(\nabla
f(U) \nabla f(U)^T) = ||\Omega||_{fro}^2$.
\end{definition}

\begin{definition}
A point $U_{*} \in \Od$ is \emph{asymptotically stable} with respect
to Algorithm (\ref{alg1}) if it has a neighborhood $\mathcal{V}$ of $U_{*}$ such that
all sequences generated by Algorithm (\ref{alg1}) with starting point
$U_0 \in \mathcal{V}$ converge to $U_{*}$.
\end{definition}

\begin{theorem}{Convergence to local optimum} \\
(1) The sequence of iterates $U_t$ of Algorithm (\ref{alg1}) satisfies:
  $\lim_{t \to \infty} ||\nabla f (U_t) || = 0$. This means that the
  accumulation points of the sequence $\{U_t\}_{t=1}^{\infty} $ are
  critical points of $f$.  \\
(2) Assume the critical points of $f$ are isolated. Let $U_{*}$ be a
critical point of $f$. Then $U_{*}$ is a local minimum of $f$ if and
only if it is asymptotically stable with regard to the sequence
generated by Algorithm (\ref{alg1}).
\end{theorem}

\begin{definition}
For an iterate $t$ of Algorithm (\ref{alg1}), and a set of indices $(i(t),j(t))$, we define the auxiliary single variable function $g_t^{ij}$ :
\begin{equation}
\label{eq:defgtij1}
g_t^{ij}(\theta) =  f\left(U_t \cdot    G(i,j,\theta)\right),
\end{equation}
\end{definition}
Note that $g_t^{ij}$ are differentiable and periodic with a period of $2 \pi$. Since $\Od$ is compact and $f$ is differentiable there exists a single Lipschitz constant $L(f) > 0$ for all $g_t^{ij}$.

\begin{theorem}{Rate of convergence} \\
Let $f$ be a continuous function with $L$-Lipschitz directional derivatives \footnote{Because $\Od$ is compact, any function $f$ with a continuous second-derivative will obey this condition.}. Let $U_t$ be the sequence generated by Algorithm \ref{alg1}. 
For the sequence of Riemannian gradients $\nabla f(U_t) \in T_{U_t} \Od$ we have: 
\begin{equation}
\underset{0 \leq t \leq T}{\operatorname{max}} E \left[ ||\nabla f(U_t)||_2^2 \right] \leq \frac{L\cdot d^2\left( f(U_0) -f_{min} \right)}{T+1} \quad .
\end{equation}
\end{theorem}
 
The proof is a Riemannian version of the proof for the rate of
convergence of Euclidean random coordinate descent for non-convex
functions \citep{patrascu2013efficient} and is provided as
supplemental material.

\section{Sparse PCA}
\label{sec:spca}
Principal component analysis (PCA) is a basic dimensionality reducing
technique used throughout the sciences. Given a data set $A \in
\reals^{d \times n}$ of $n$ observations in $d$ dimensions, the
principal components are a set of orthogonal vectors $z_1, z_2, \ldots
, z_m \in \reals^d$, such that the variance $\sum_{i=1}^m
z_i^T A A^T z_i$ is maximized. The data is then represented in a new
coordinate system $\hat{A} = Z^T A$ where $Z = [z_1, z_2, \ldots, z_m]
\in \reals^{d \times m}$. 

One drawback of ordinary PCA is lack of interpretability. In the
original data $A$, each dimension usually has an understandable
meaning, such as the level of expression of a certain gene. The
dimensions of $\hat{A}$ however are typically linear combinations of
all gene expression levels, and as such are much more difficult to
interpret. A common approach to the problem of finding
\emph{interpretable} principal components is Sparse PCA \citep{zou2006sparse, journee2010generalized,d2007direct,zhang2012sparse,zhang2012large}.
SPCA aims to find vectors $z_i$ as in PCA,
but which are also sparse. In the gene-expression example, the non-zero
components of $z_i$ might correspond to a few genes that explain well
the structure of the data $A$.

One of the most popular approaches for solving the problem of finding
sparse principal components is the work by
\citet{journee2010generalized}. In their paper, they formalize the
problem as finding the optimum of the following constrained
optimization problem to find the sparse basis vectors $Z$:
\begin{eqnarray}
\label{eq:spca}
\underset{U \in \reals^{n \times m}, Z \in \reals^{d \times m} }{\operatorname{argmax}} Tr(Z^T A U) - \gamma \sum_{ij} |Z_{ij}| \\ 
s.t. \; \; U^T U = I_m , \sum_{i=1}^d Z_{ij}^2 = 1 \; \forall j=1 \ldots m  \nonumber \quad.
\end{eqnarray}

Journ{\'e}e et al. provide an algorithm to solve Eq. \ref{eq:spca}
that has two parts: The first and more time consuming part finds an
optimal $U$, from which optimal $Z$ is then found.  We focus here on
the problem of finding the matrix $U$. Note that when $m = n$, the
constraint $ U^T U = I_m $ implies that $U$ is an orthogonal matrix.

We use a second formulation of the optimization problem, also given by
Journ{\'e}e et al. in section 2.5.1 of their paper:
\begin{equation}
  \label{eq:spca2}
  \begin{aligned}
    \underset{U \in \reals^{n \times m }}{\operatorname{argmax}} &\sum_{j=1}^m \sum_{i=1}^d     [|(A \cdot U)_{ij}| - \gamma]_{+}^2 \\ 
    &s.t. \; \; U^T U = I_m , \nonumber
  \end{aligned}
\end{equation}
where $n$ is the number of samples, $d$ is the input dimensionality
and $m$ is output dimension (the number of PCA component
computed). This objective is once-differentiable and the objective
matrix $U$ grows with the number of samples $n$.

\subsection{Givens rotation algorithm for the full case $m = n$}
If we choose the number of principal components $m$ to be equal to the
number of samples $n$ we can apply Algorithm ((\ref{alg1})) directly to
solve the optimization problem of Eq. \ref{eq:spca2}.  Explicitly, at
each round $t$, for choice of coordinates $(i,j)$ and a matrix $U_t
\in \Od$, the resulting coordinate minimization problem is:
\begin{equation}
\begin{aligned}
\label{eq:spcatheta}
&\underset{\theta}{\operatorname{argmin}} -\sum_{j=1}^m \sum_{i=1}^d    [|(AU_tG(i,j,\theta))_{ij}| - \gamma]_{+}^2 = \\
&\underset{\theta}{\operatorname{argmin}} -\sum_{k=1}^d [|cos(\theta) (AU_t)_{ki} + sin(\theta) (AU_t)_{kj}| - \gamma]_{+}^2 +  \\
& \quad \quad \quad \quad \quad [|-sin(\theta) (AU_t)_{ki} + cos(\theta) (AU_t)_{kj}| - \gamma]_{+}^2
\end{aligned}
\end{equation}

 \begin{algorithm}
\caption{Riemannian coordinate minimization for sparse PCA}
\label{alg2}
\begin{algorithmic}
\REQUIRE Data matrix $A \in \reals^{d \times n}$, initial matrix $U_0 \in \On$,
sparsity parameter $\gamma \geq 0$
\STATE $t=0$
\STATE $AU = A \cdot U_0$ .
\WHILE{not converged}
\STATE 1. Sample uniformly at random a pair $(i(t),j(t))$ such that $1 \leq i(t) < j(t) \leq n$.
\STATE 2. $\theta_{t+1} = \underset{\theta }{\operatorname{argmax}}$ \\ $\sum_{k=1}^d ([|cos(\theta) (AU)_{ki(t)} + sin(\theta) (AU)_{kj(t)}| - \gamma]_{+}^2$ \\ $+  [|-sin(\theta) (AU)_{ki(t)} + cos(\theta) (AU)_{kj(t)}| - \gamma]_{+}^2 )$.
\STATE 3.$AU = AU \cdot G(i(t),j(t)),\theta_{t+1})$.
\STATE 4. $t = t+1$.
\ENDWHILE
\STATE 5. $Z = solveForZ(AU,\gamma)$  // Algorithm 6 of \\  \quad \citet{journee2010generalized}.
\ENSURE $Z \in \reals^{d \times n}$
\end{algorithmic}
\end{algorithm}

See Algorithm (\ref{alg2}) for the full procedure. In practice, there
is no need to store the matrices $U_t$ in memory, and one can work
directly with the matrix $A U_t$.  Evaluating the above expression
\ref{eq:spcatheta} for a given $\theta$ requires $O(d)$ operations,
where $d$ is the dimension of the data instances. We found in practice
that optimizing Eq. \ref{eq:spcatheta} required an order of 5-10 evaluations. Overall each
iteration of Algorithm (\ref{alg2}) requires $O(d)$ operations.

\subsection{Givens rotation algorithm for the case $m < n$}
 The major drawback of Algorithm (\ref{alg2}) is that it requires the
 number of principal components $m$ to be equal to the number of
 samples $n$.  This kind of ``full dimensional sparse PCA'' may not be
 necessary when researchers are interested to obtain a small number of
 components. We therefore develop a streaming version of Algorithm
 (\ref{alg2}). For a small given $m$, we treat the data as if only $m$
 samples exist at any time, giving an intermediate model $AU \in
 \reals^{d \times m}$.  After a few rounds of optimizing over this
 subset of samples, we use a heuristic to drop one of the previous
 samples and incorporate a new sample. This gives us a streaming
 version of the algorithm because in every phase we need only $m$
 samples of the data in memory. The full details of the algorithm are
 given in the supplemental material.
 

\subsection{Experiments}

Sparse PCA attempts to trade-off two variables: the fraction of data
variance that is explained by the model's components, and the level of
sparsity of the components.  In our experiment, we monitor a third
important parameter, the number of floating point operations (FLOPS)
performed to achieve a certain solution. To compute the number of
FLOPS we counted the number of additions and multiplications computed
on each iteration. This does not include pointer arithmetic.

We first examined Algorithm \ref{alg2} for the case where $m=n$.
We used the prostate cancer gene expression data by \citet{singh2002gene}.
This dataset consists of the gene expression levels for 52 tumor and 50 normal samples over
12,600 genes, resulting in a $12,600 \times 102$ data matrix.

We compared the performance of our approach with that of the {\em
Generalized Power Method} of \citet{journee2010generalized}. We focus
on this method for comparisons because both methods optimize the same
objective function, which allows to characterize the relative
strengths and weaknesses of the two approaches.

As can be seen in Figure \ref{fig:spca_full}, the Givens coordinate minimization method
finds a sparser solution with better explained variance, and does so faster than the generalized power method.

\begin{figure}
\begin{center}

\centerline{\includegraphics[width=0.5\linewidth]{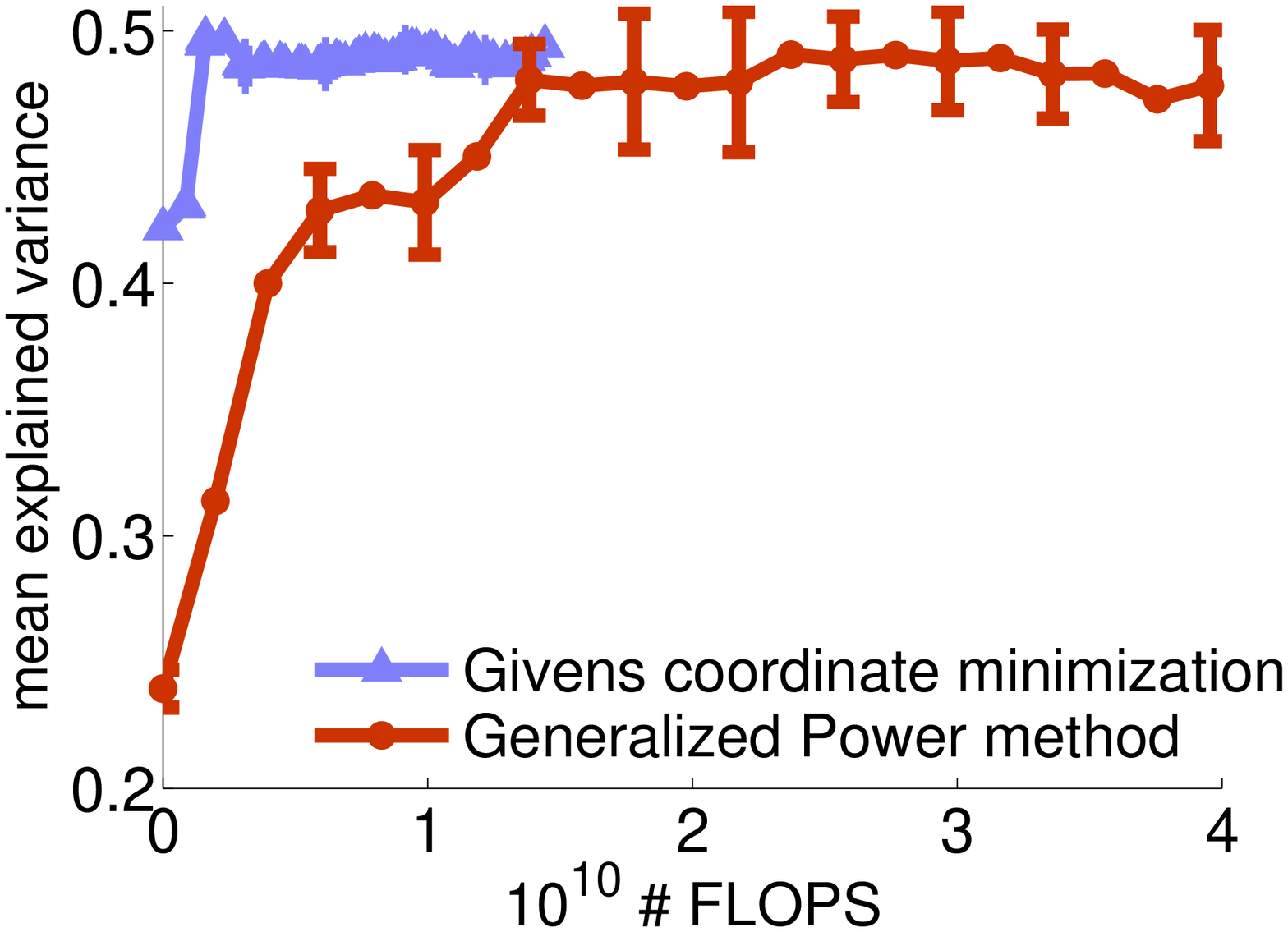} \includegraphics[width=0.5\linewidth]{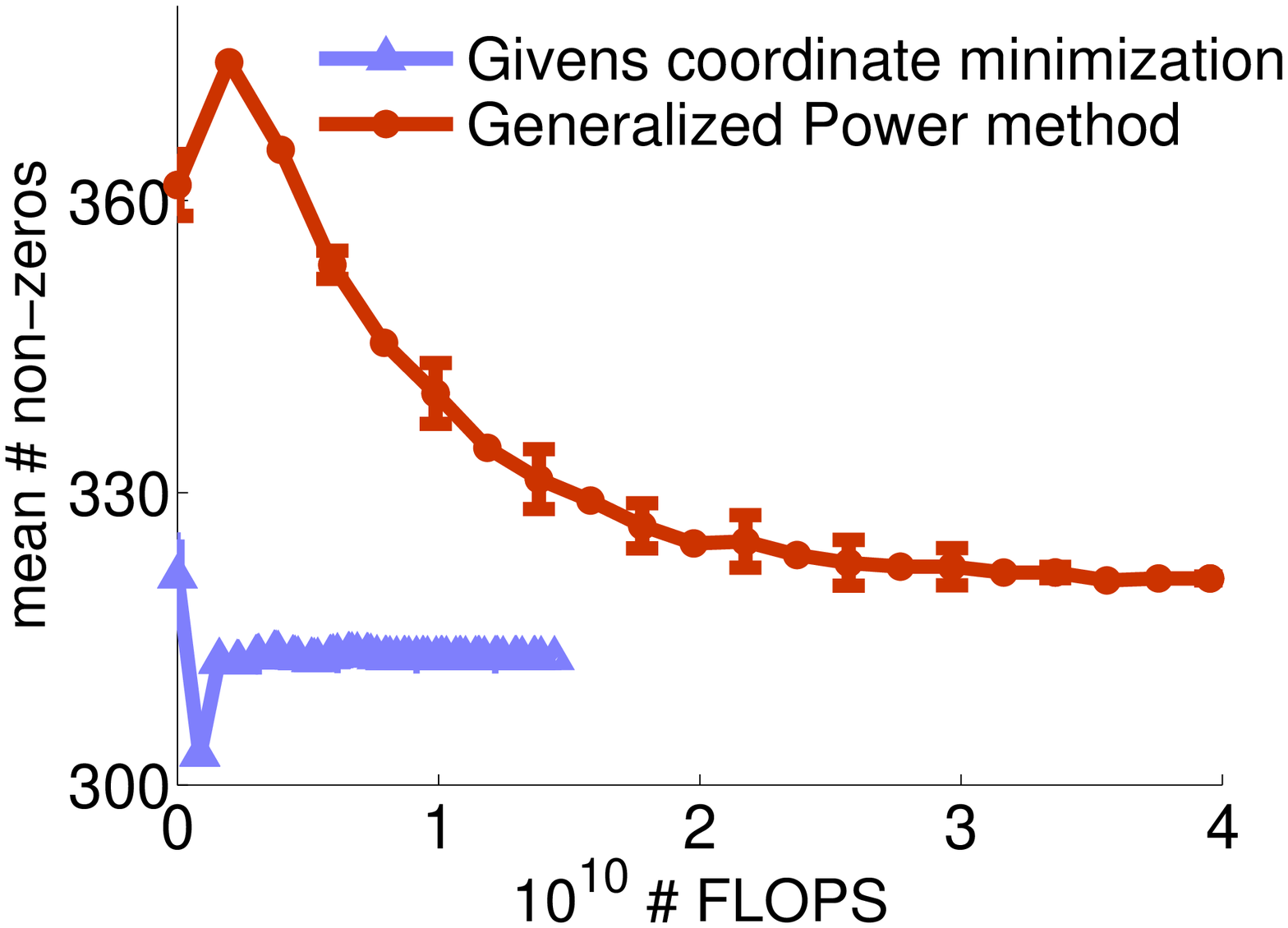} }
\centerline{{ (a) explained variance  } \hskip 0.3in { (b) number of non-zeros }}
\begin{small}\caption{(a) The explained variance as function of FLOPS of the coordinate minimization method from Algorithm \ref{alg2} and 
of the generalized power method by \citet{journee2010generalized}, on a prostate cancer gene expression dataset. (b) The number of non-zeros
in the sparse PCA matrix as function of FLOPS of the coordinate minimization method from Algorithm \ref{alg2} and 
of the generalized power method by \citet{journee2010generalized}, on a prostate cancer gene expression dataset. The size of the sparse PCA matrix is $12,600 \times 102$.} \label{fig:spca_full}\end{small}

\end{center}
\vskip -0.05in
\end{figure}

We tested the streaming version of the coordinate descent algorithm
for sparse PCA (Algorithm 5, supp. material) on a recent large gene
expression data set collected from of six human brains
\citep{hawrylycz2012anatomically}. Overall, each of the 20K human genes
was measured at 3702 different brain locations, and this data can be
used to study the spatial patterns of mRNA expression across the human
brain.

We again compared the performance of our approach with that of the {\em
Generalized Power Method} of \citet{journee2010generalized}. 

We split the data into 5 train/test partitions, with each train set
including 2962 examples and each test set including 740 examples. We
evaluated the amount of variance explained by the model on the test
set. We use the adjusted variance procedure suggested in this case by
\citet{zou2006sparse}, which takes into account the fact that the
sparse principal components are not orthogonal.

For the Generalized Power Method we use the greedy $l_1$ version of
\citet{journee2010generalized}, with the parameter $\mu$ set to 1. We
found the greedy version to be more stable and to be able to produce
sparse solutions when the number of components was $m>1$.  We used
values of $\gamma$ ranging from $0.01$ to $0.2$, and two stopping
conditions: ``convergence'', where the algorithm was run until its
objective converged within a relative tolerance level of $10^{-4}$,
and ``early stop'' where we stopped the algorithm after 14\% of the
iterations required for convergence.

For our algorithm we used the same range of $\gamma$ values, and used
an ``early stop'' condition where the algorithm was stopped after using
14\% of the samples.

Figure \ref{fig:scatter_all} demonstrates the tradeoff between floating
point operations and explained variance for sparse PCA with 3, 5
and 10 components and with 3 sparsity levels: 5\%, 10\% and 20\%. Using low
dimensions is often useful for visual exploration of the data.  Each
dot represents one instance of the algorithm that was run with a
certain value of $\gamma$ and stopping criterion. To avoid clutter we
only show instances which performed best in terms of explained
variance or few FLOPS.  

When strong sparsity is required (5\% or 10\% sparsity), the
givens-rotation coordinate descent algorithm finds solutions faster
(blue rectangles are more to the left in Figure
\ref{fig:scatter_all}), and these solutions are similar or better in
terms of explained variance.

For low-dimensional less sparse solutions (20\% sparsity) we find that
the generalized power method finds comparable or better solutions
using the same computational cost, but only when the number of
components is small, as seen in Figure \ref{fig:scatter_all}.c,f,i.

%

\begin{figure*}[!t]
  \vskip -0.07in
  \centering
  \begin{center}
    \centerline{{ (a) max. sparsity 5\% } \hskip 0.35in { (d)  max. sparsity 5\%  } \hskip 0.35in { (g)  max. sparsity 5\%  }}
    \centerline{
      \includegraphics[width=0.25\linewidth]{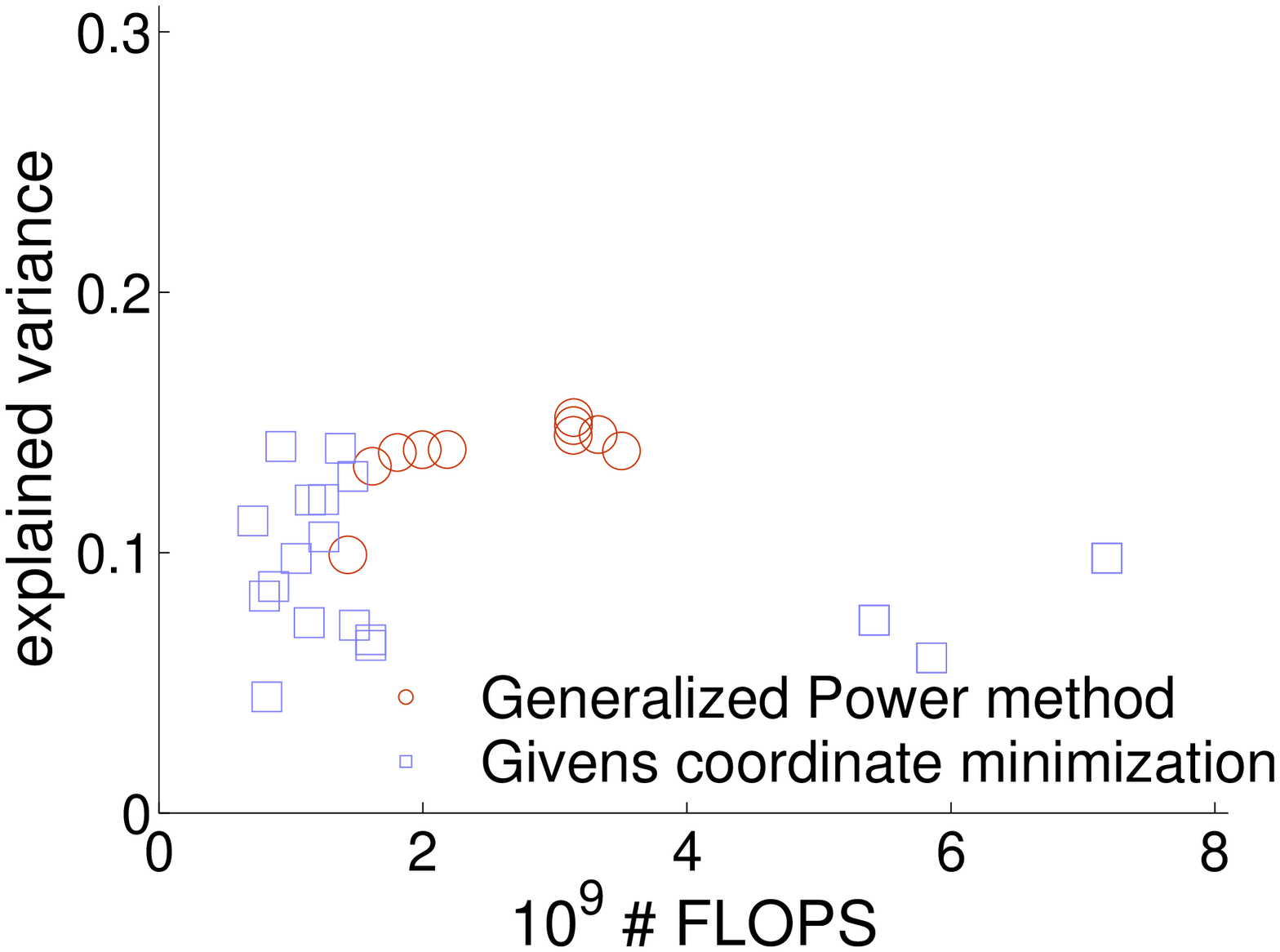}  
      \includegraphics[width=0.25\linewidth]{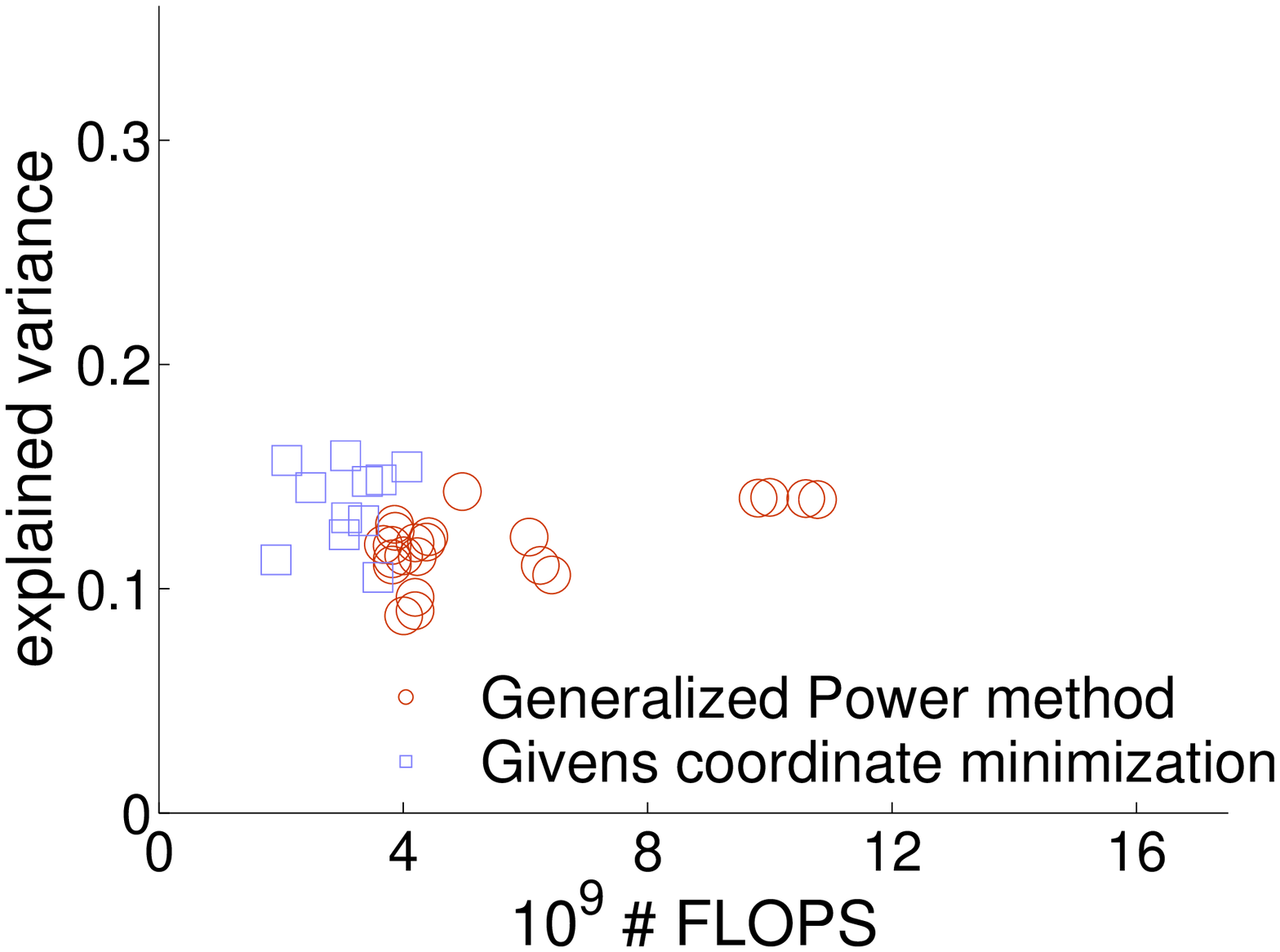} 
      \includegraphics[width=0.25\linewidth]{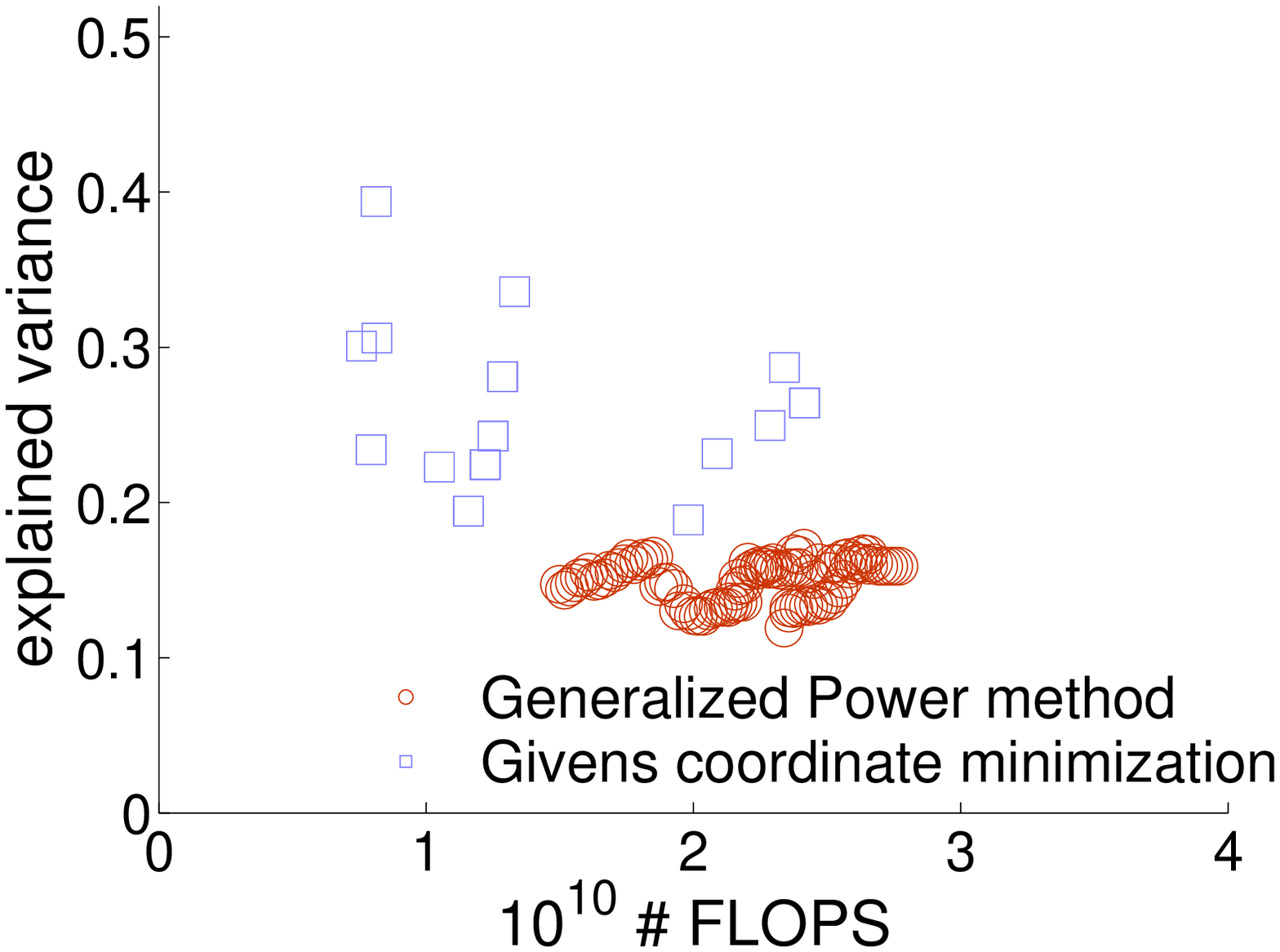}
    }
    \centerline{{ (b) max. sparsity 10\% } \hskip 0.35in { (e) max. sparsity 10\% } \hskip 0.35in { (h) max. sparsity 10\% }}
    
    \centerline{
      \includegraphics[width=0.25\linewidth]{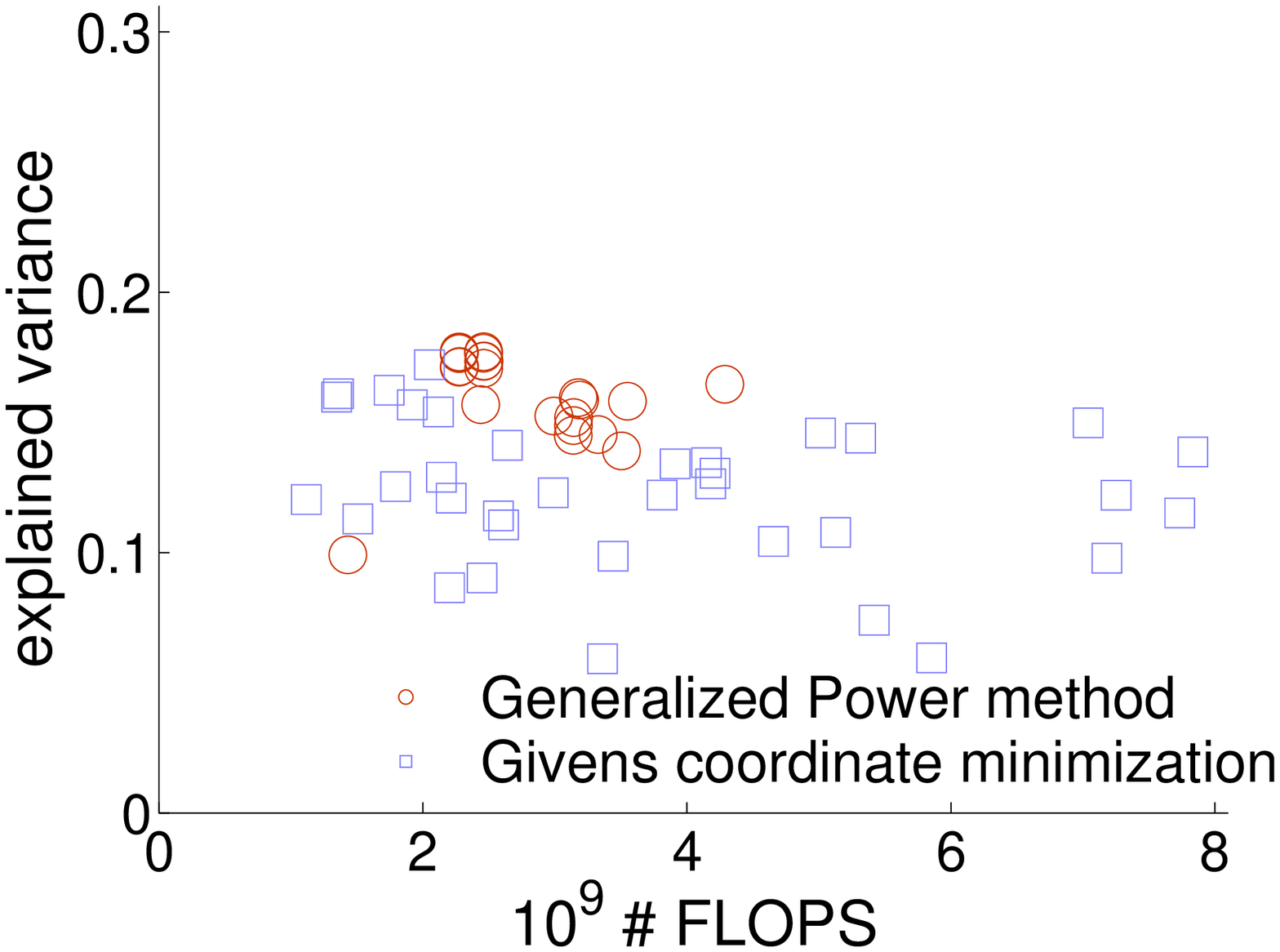}  
      \includegraphics[width=0.25\linewidth]{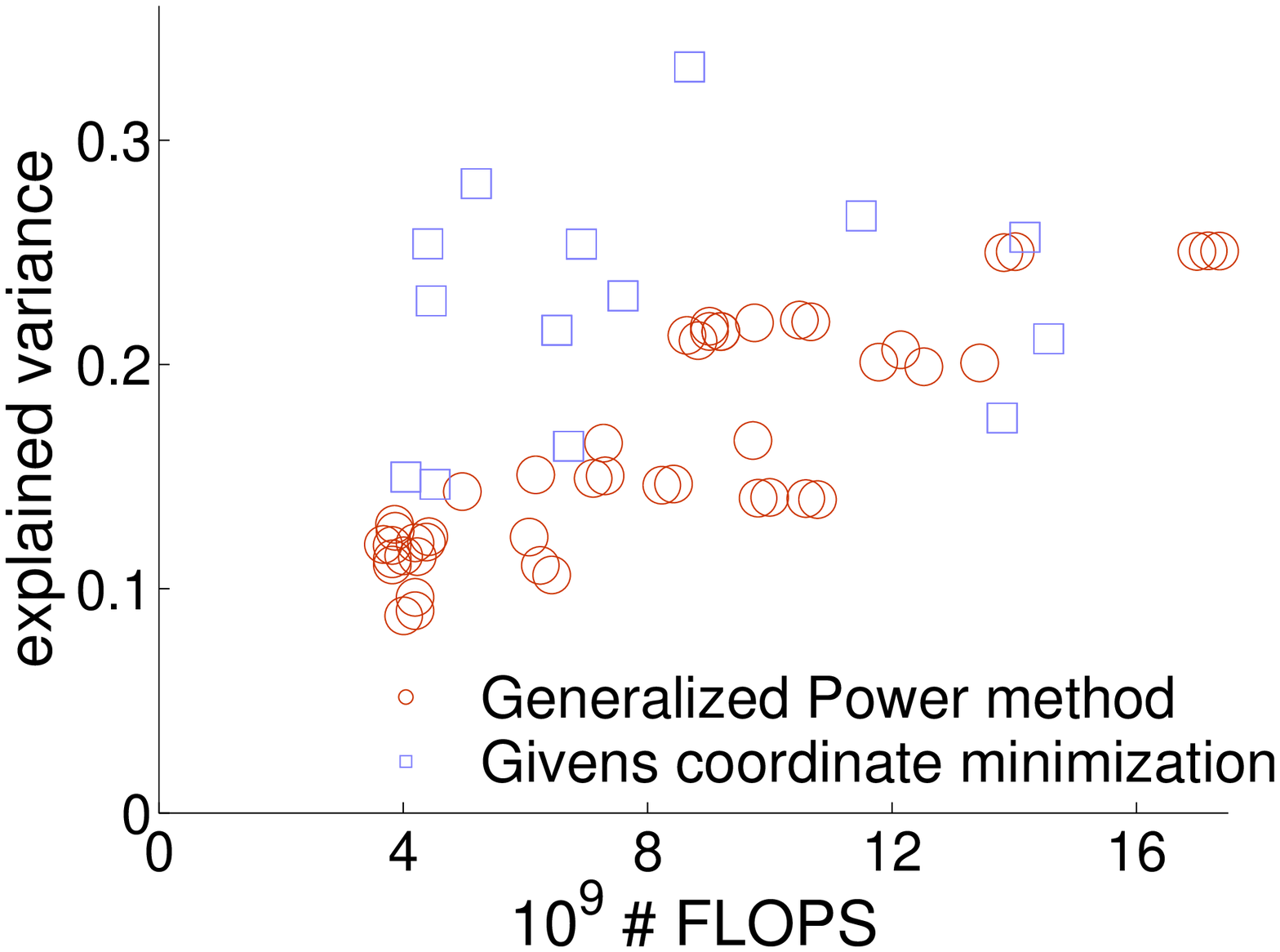} 
      \includegraphics[width=0.25\linewidth]{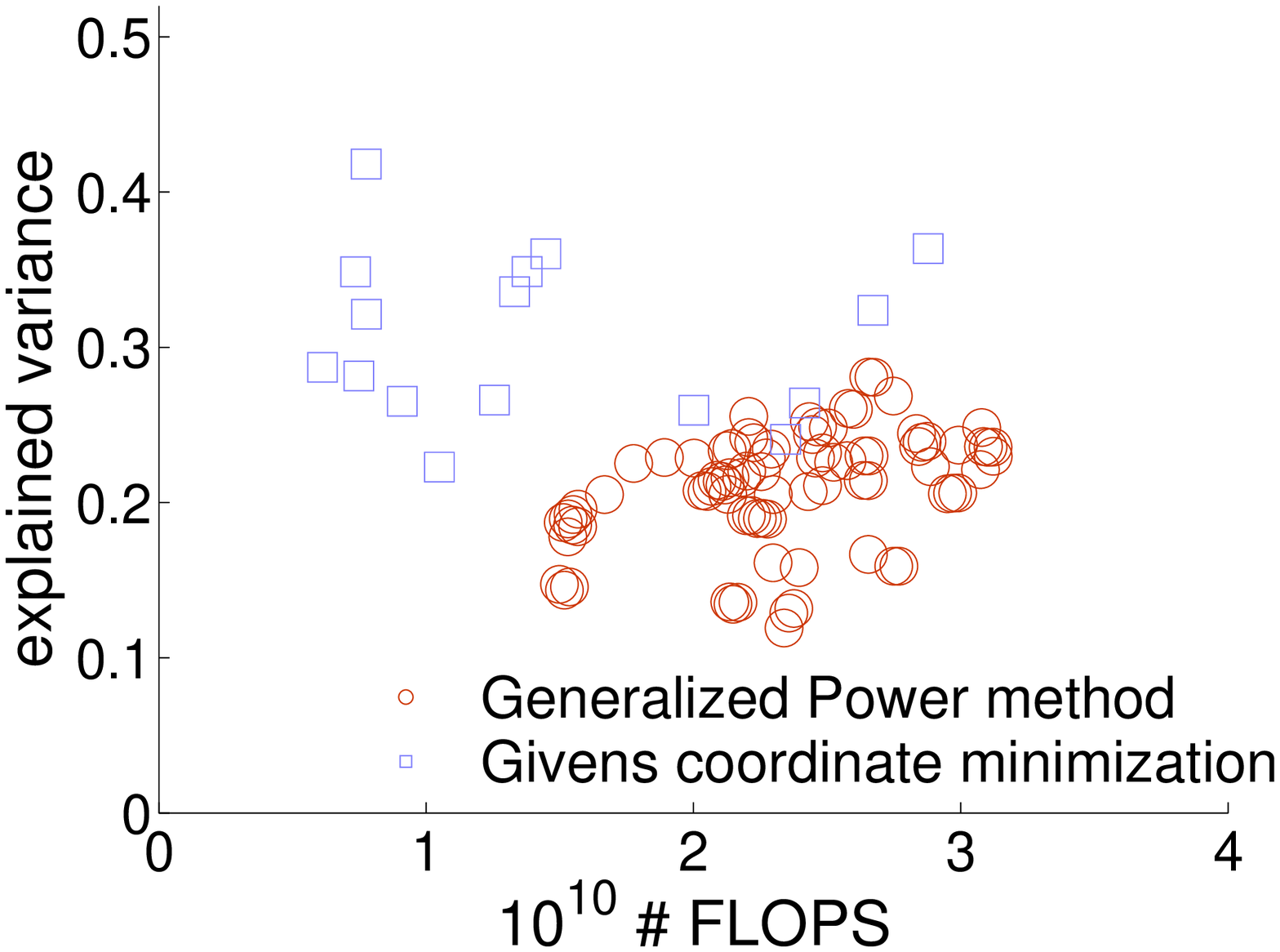}
    }
    \centerline{{ (c) max. sparsity 20\% } \hskip 0.35in { (f) max. sparsity 20\% } \hskip 0.35in { (i) max. sparsity 20\% }}
    
    \centerline{
      \includegraphics[width=0.25\linewidth]{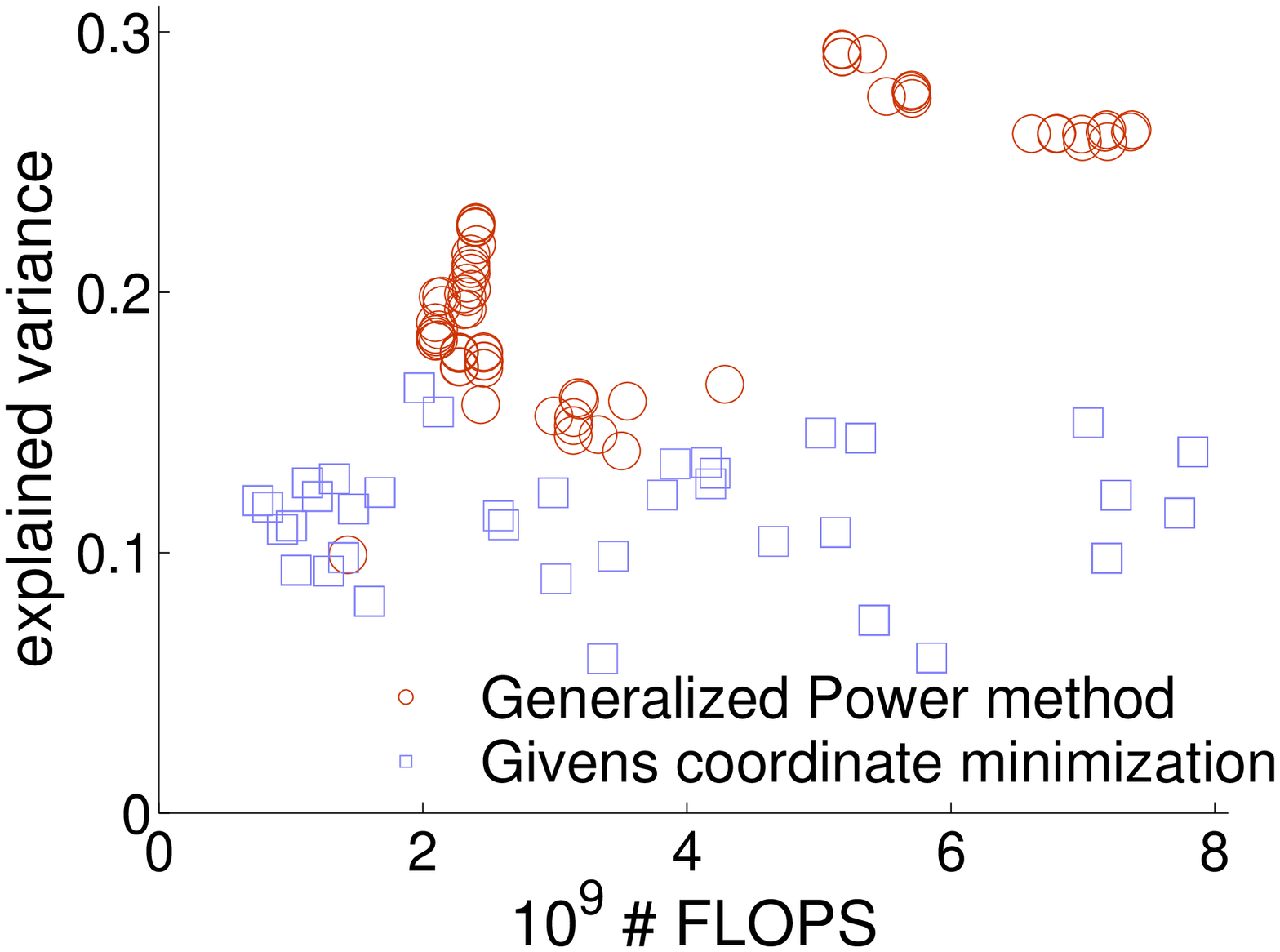} 
      \includegraphics[width=0.25\linewidth]{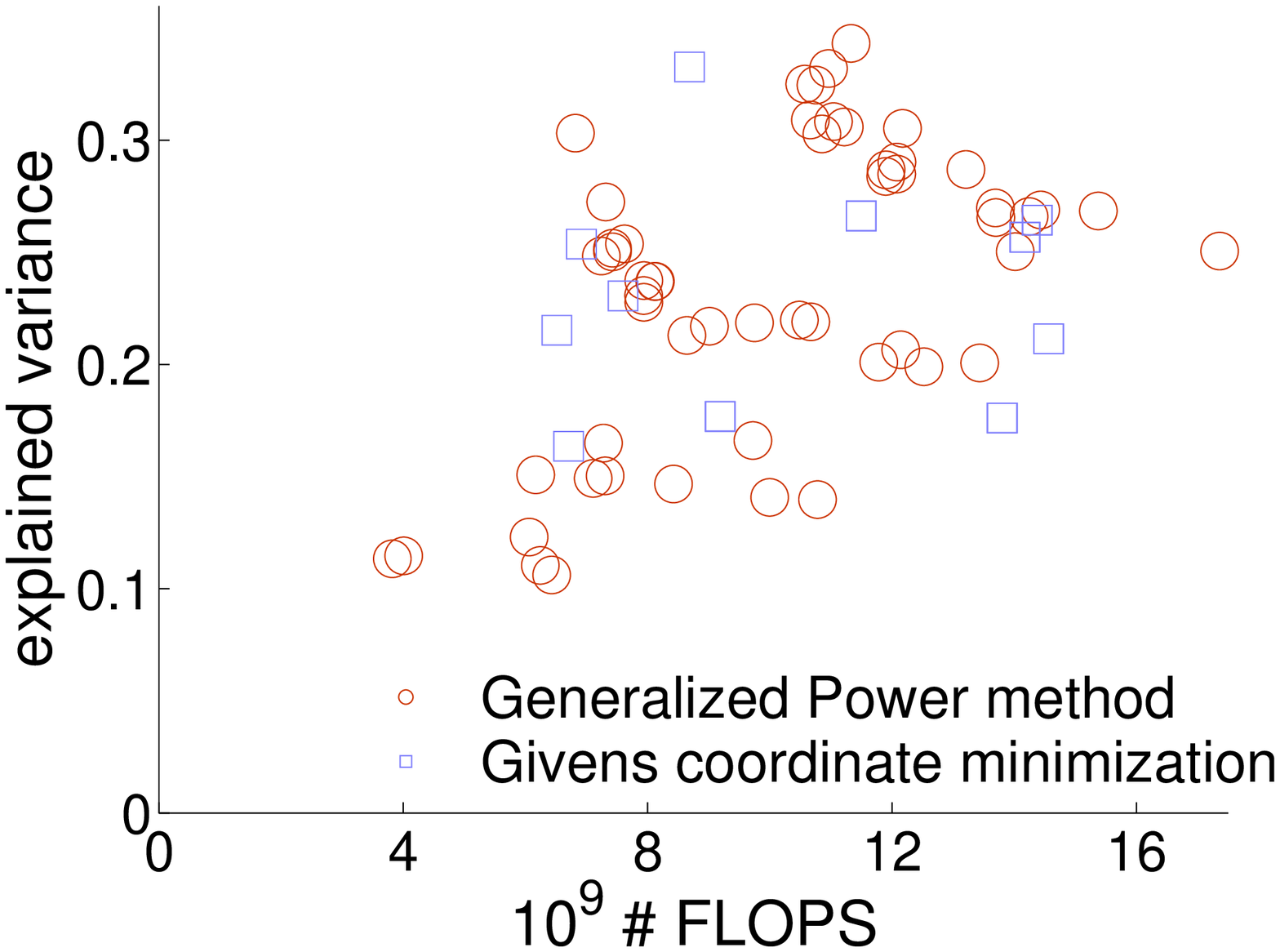}  
      \includegraphics[width=0.25\linewidth]{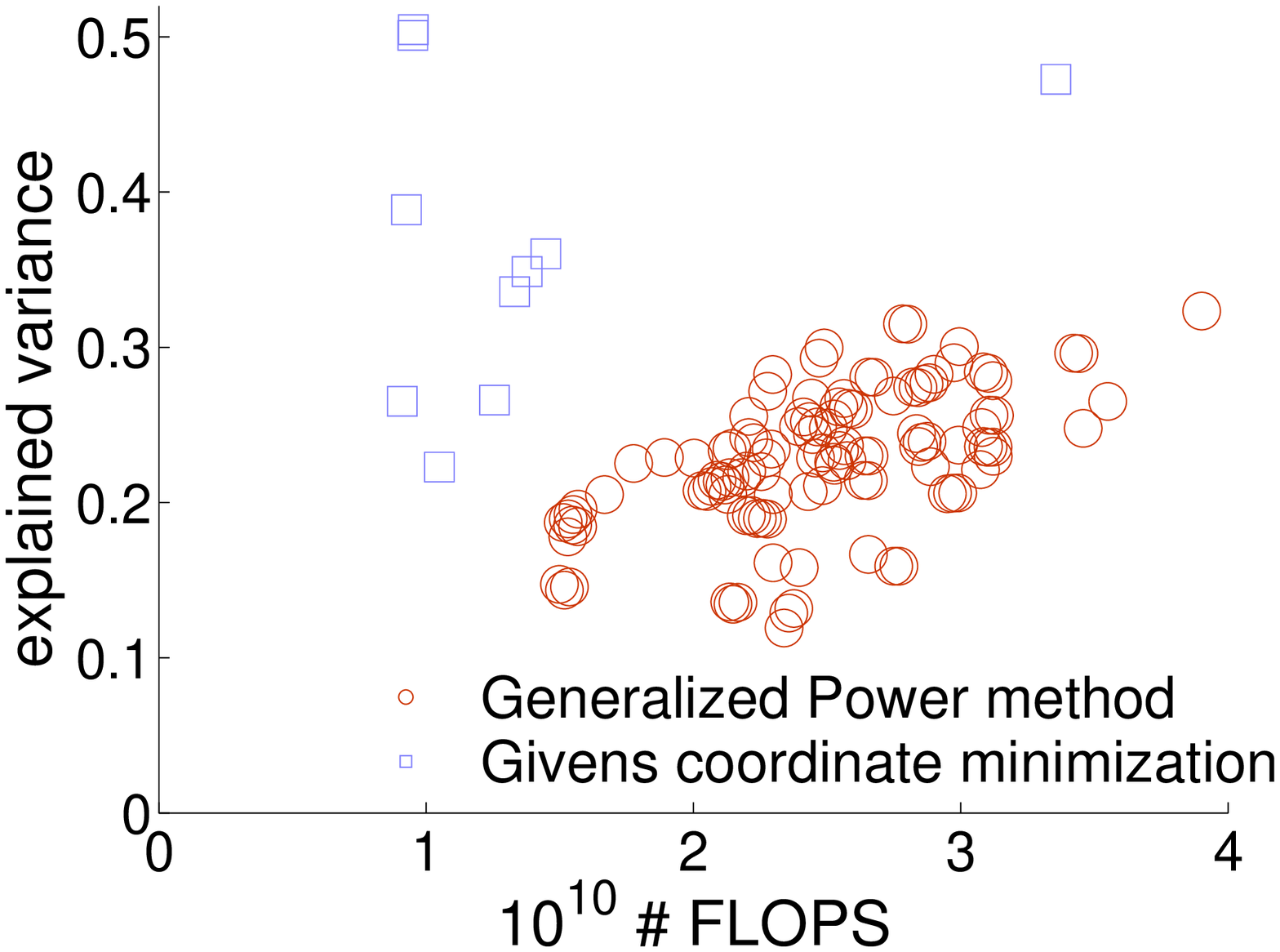} 
    }
    \centerline{{ 3 components } \hskip 0.8in {  5 components }  \hskip 0.8in {  10 components }}
    \begin{small}
      \caption{The tradeoff between explained variance and
        computational cost for 3, 5 and 10-component sparse-PCA models
        applied to Human gene expression data.  The models are
        constrained for maximum sparsity of 5\% (a), (d) \& (f), 10\%
        (b), (e) \& (h) and 20\% (c), (f) \& (i).  Red circles are
        instances of the Generalized Power method
        \citet{journee2010generalized}; Blue squares represent the
        Givens coordinate descent procedure. Both algorithms were run
        with $\gamma$ in the range [0.01,...,0.2] and two stopping
        criteria ('early-stop' and 'convergence'). The points
        presented are chosen for best performance in terms of
        computational cost or explained variance. Explained variance
        was adjusted following \citet{zou2006sparse} }
      \label{fig:scatter_all}
    \end{small}
  \end{center}
\end{figure*} 
\vskip -0.05in

\section{Orthogonal tensor decomposition}
\label{sec:tensor}
Recently it has been shown that many classic machine learning problem
such as Gaussian Mixture Models and Latent Dirichlet Allocation can be
solved efficiently by using 3rd order moments
\citep{anandkumar2012tensor,hsu2013learning,anandkumar2012spectral,anandkumar2012method,chaganty2013spectral}. These methods ultimately
rely on finding an orthogonal decomposition of 3-way tensors $T \in
\reals^{d\times d \times d}$, and reconstructing the solution from the
decomposition.  In this section, we show that the problem of finding
an orthogonal decomposition for a tensor $T \in \reals^{d\times d
\times d}$ can be naturally cast as a problem of optimization over the
orthogonal matrix manifold. We then apply Algorithm (\ref{alg1}) to this
problem, and compare its performance on a task of finding a Gaussian
Mixture Model with a state-of-the-art tensor
decomposition method, namely the robust Tensor Power Method
\citep{anandkumar2012tensor}. We find that the Givens coordinate
minimization method consistently finds better solutions when the
number of mixture components is large.

\subsection{Orthogonal tensor decomposition}
The problem of tensor decomposition is very hard in general
\citep{kolda2009tensor}. However, a certain class of tensors known as
``orthogonally decomposable'' tensors are easier to decompose, as has
been demonstrated recently by \citet{anandkumar2012tensor,
hsu2013learning} and others.  In this section, we introduce the problem
of orthogonal tensor decomposition, and provide a new characterization
of the solution to the tensor-decomposition problem as the solution of
an optimization problem on the orthogonal matrix manifold.

The resulting algorithm is similar to one recently proposed by \citet{ishteva2013jacobi}.
However, we aim for full diagonalization, while they focus on finding a good low-rank approximation.
This results in different objective functions: ours involves third-order polynomials on $\Od$, while Ishteva et al.'s
results in sixth-order polynomials on the low-rank compact Stiefel manifold. 
Diagonalizing the tensor $T$ is attainable in our case thanks to the strong assumption that it is orthogonally decomposable.
Nonetheless, both methods are extensions of Jacobi's eigenvalue algorithm to the tensor case, in different setups.

We start with preliminary notations and definitions.  We focus
here on symmetric tensors $T \in \mathds{R}^{d \times d \times d}$.  A
third-order tensor is symmetric if its values are identical for any
permutation $\sigma$ of the indices: with $T_{i_1 i_2 i_3} =
T_{i_{\sigma(1)} i_{\sigma(2)} i_{\sigma(3)}}$.

We also view a tensor $T$ as a trilinear map. \\ $T : \reals^d
 \times \reals^d \times \reals^d \rightarrow \reals$: $T(v_1,v_2,v_3)
 = \sum_{a,b,c = 1}^d T_{abc} v_{1a} v_{2b} v_{3c}$.

Finally, we also use the three-form tensor product of a vector $u \in
\reals^d$ with itself: $u \otimes u \otimes u \in \reals^{d \times d
\times d}$, $(u \otimes u \otimes u)_{abc} = u_a \cdot u_b \cdot
u_c$. Such a tensor is called a \textit{rank-one} tensor.

Let $T \in \reals^{d \times d \times d}$ be a symmetric tensor.
\begin{definition}
\label{def:orth_dec}
A tensor $T$ is \textit{orthogonally decomposable} if there exists an
orthonormal set of vectors $v_1, \ldots v_d \in \reals^d$, and positive scalars $\lambda_1,
\ldots \lambda_d >0$ such that:
\begin{equation}
\label{eq:orth_dec}
T = \sum_{i=1}^d \lambda_i (v_i \otimes v_i \otimes v_i).
\end{equation}
\end{definition}

Unlike matrices, most symmetric tensors are not orthogonally
decomposable. However, as shown by \citet{anandkumar2012tensor,
hsu2013learning,anandkumar2013tensor}, several problems of interest,
notably Gaussian Mixture Models and Latent Dirichlet Allocation do
give rise to third-order moments which are orthogonally decomposable
in the limit of infinite data.

The goal of orthogonal tensor decomposition is, given an orthogonally
decomposable tensor $T$, to find the orthogonal vector set $v_1,
\ldots v_d \in \reals^d$ and the scalars $\lambda_1, \ldots \lambda_d
>0$.

We now show that finding an orthogonal decomposition can be stated as
an optimization problem over $\Od$:

\begin{theorem}
  \label{th:unq}
  Let $T \in R^{d \times d \times d}$ have an orthogonal decomposition
  as in Definition \ref{def:orth_dec}, and consider the optimization
  problem
  \begin{equation}
    \label{eq:obj}
    \underset{U \in \Od}{\operatorname{max}} f(U) = \sum_{i=1}^d T(u_i,u_i,u_i),
  \end{equation}
  where $U = [u_1 \, u_2 \, \ldots \, u_d]$.  The stable stationary
  points of the problem are exactly orthogonal matrices $U$ such that
  $u_i = v_{\pi(i)}$ for a permutation $\pi$ on $[d]$. The maximum
  value they attain is $\sum_{i=1}^d \lambda_i$.
\end{theorem}

The proof is given in the supplemental material.

\subsection{Coordinate minimization algorithm for orthogonal tensor decomposition}

We now adapt Algorithm (\ref{alg1}) for solving the problem of
orthogonal tensor decomposition of a tensor $T$, by minimizing the
objective function \ref{eq:obj}, $f(U) = - \sum_{i=1}^d
T(u_i,u_i,u_i)$.  For this we need to calculate the form of the
function $f\left(U \cdot G(i,j,\theta)\right)$.  Define $\tilde{u}_i =
cos(\theta)u_i + sin(\theta) u_j$ and $\tilde{u}_j = cos(\theta)u_j -
sin(\theta) u_i$.
\begin{align*}
& f\left(U \cdot 	G(i,j,\theta)\right) = \sum_{k \neq i,j}^d T(u_k,u_k,u_k) + \\
& T\left(\tilde{u}_i , \; \tilde{u}_i ,\; \tilde{u}_i \right) +  
 T\left(\tilde{u}_j, \; \tilde{u}_j,\; \tilde{u}_j\right).
\end{align*}
Define:
\begin{equation}
\label{eq:defgtij2}
g_t^{ij}(\theta) =  f\left(U \cdot 	G(i,j,\theta)\right),
\end{equation}
and denote by $\tilde{T}$ the tensor such that:
\begin{equation}
\label{eq:Tbar}
\tilde{T}_{ijk} = T(u_i,u_j,u_k).
\end{equation}

Collecting terms, using the symmetry of $T$ and some basic trigonometric identities, we then have:
\begin{align}
\label{eq:gtij3}
g_t^{ij}(\theta) = 
 &cos^3(\theta)\left(\tilde{T}_{iii}+\tilde{T}_{jjj} - 3\tilde{T}_{ijj} - 3\tilde{T}_{jii}\right) \\ \nonumber 
+&sin^3(\theta)\left(\tilde{T}_{iii}-\tilde{T}_{jjj} -3\tilde{T}_{ijj} +3\tilde{T}_{jii}\right)  \\ \nonumber
+&cos(\theta)\left(3\tilde{T}_{ijj}+3\tilde{T}_{jii}\right) \\ \nonumber
+&sin(\theta)\left(3\tilde{T}_{ijj} -3\tilde{T}_{jii} \right).
\end{align}

In each step of the algorithm, we maximize $g_t^{ij}(\theta)$ over
$-\pi \leq \theta < \pi$. The function $g_t^{ij}$ has at most 3 maxima
that can be obtained in closed form solution, and thus $g_t^{ij}$ can
be maximized in constant time.

\begin{algorithm}
  \caption{Riemannian coordinate maximization for orthogonal tensor decomposition}
  \label{alg:tens}
  \begin{algorithmic}
    \REQUIRE Symmetric tensor $T \in \mathds{R}^{d \times d \times d}$, initial matrix $U_0 \in \Od$
    
    \STATE $t=0$
    \WHILE{not converged}
    \STATE 1. Sample uniformly at random a pair $(i(t),j(t))$ such that $1 \leq i(t) < j(t) \leq d$.
    \STATE 2. Calculate $\tilde{T}_{iii}$, $\tilde{T}_{jjj}$, $\tilde{T}_{ijj}$, $\tilde{T}_{jii}$ as in \ref{eq:Tbar}.
    \STATE 3. $\theta_t = \underset{\theta }{\operatorname{argmax}} \; g_t^{ij}(\theta)$, where $g_t^{ij}$ is defined as in \ref{eq:gtij3}.
    \STATE 4. $U_{t+1} = U_t G(i,j,\theta_t)$.
    \STATE 5. $t = t+1$.
    \ENDWHILE
  \end{algorithmic}
\end{algorithm}
The most computationally intensive part of Algorithm \ref{alg:tens} is
line 2, naively requiring $O(d^3)$ operations. This can be improved to
$O(d^2)$ per iteration, with a one-time precomputation of $O(d^4)$
operations, by maintaining an auxiliary tensor in memory. The more
efficient algorithm is not described due to space constraints. We will
make the code available online.

\subsection{Experiments}

\citet{hsu2013learning} and \citet{anandkumar2012tensor} have recently
shown how the task of fitting a Gaussian Mixture Model (GMM) with
common spherical covariance can be reduced to the task of orthogonally
decomposing a third moment tensor.  We evaluate the Givens coordinate
minimization algorithm using this task. We compare with a state of the
art tensor decomposition method, the robust tensor power method, as
given in \citet{anandkumar2012tensor}.

\begin{figure}
  \vskip -0.1in
  \begin{center}
    \centerline{\includegraphics[width=0.5\linewidth]{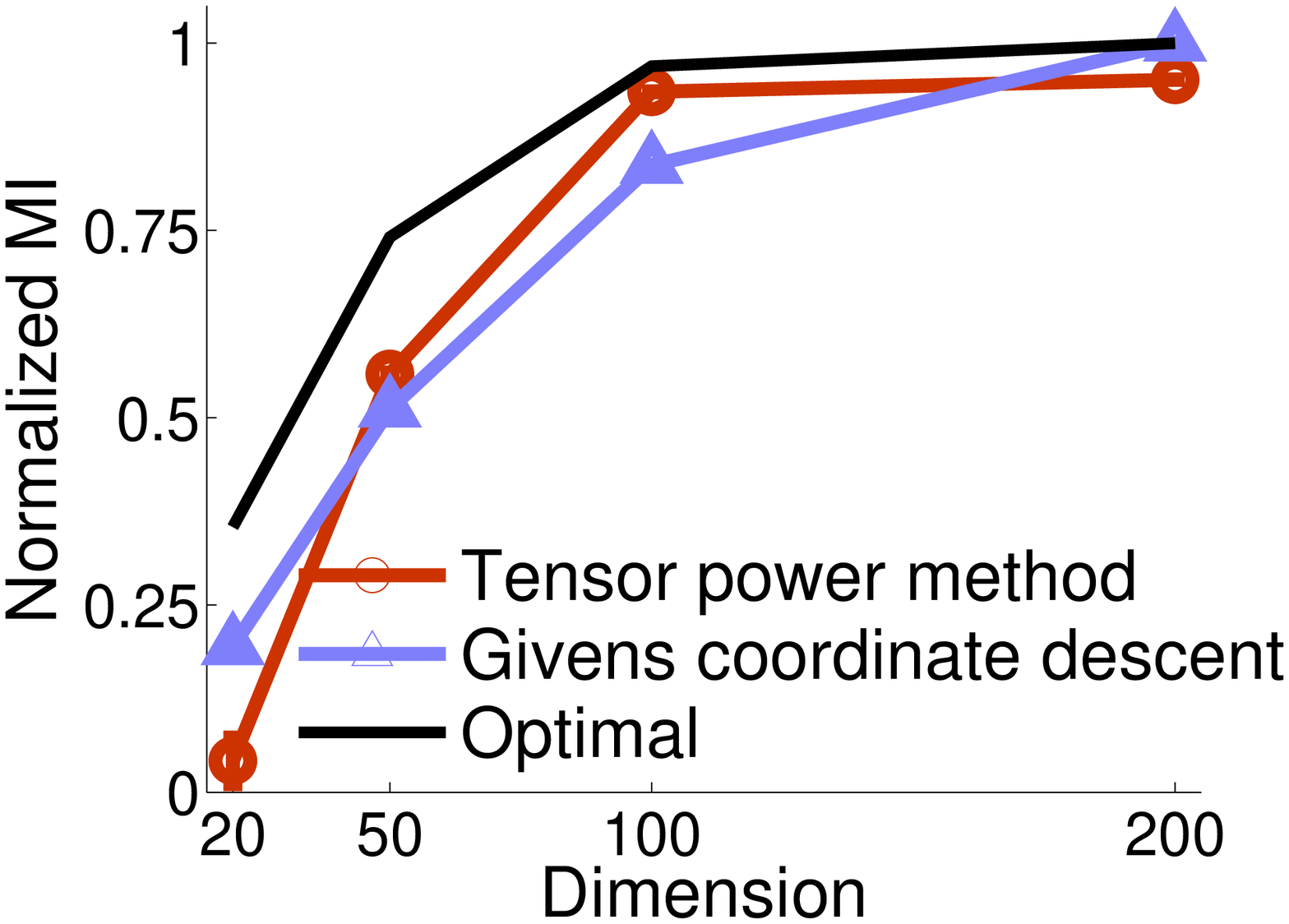}  \includegraphics[width=0.5\linewidth]{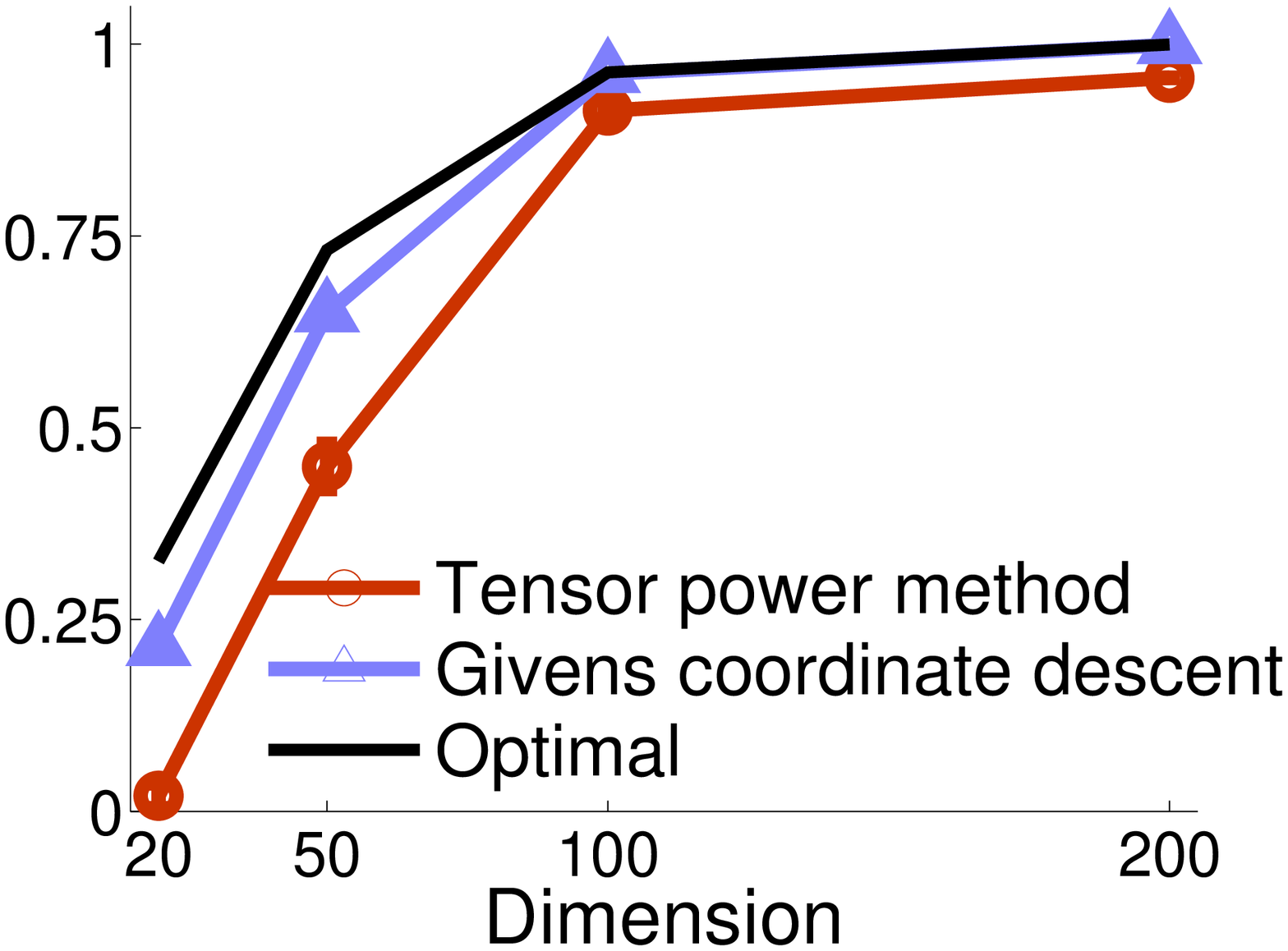}}
    \centerline{{ (a) 10,000 samples } \hskip 0.3in { \; \;(b)  200,000 samples  }}
    \begin{small}
      \caption{Clustering performance in terms of normalized MI of the
      Givens coordinate algorithm vs. the tensor power method of
      \citet{anandkumar2012tensor}. Clustering by fitting a GMM from
      samples drawn from a 20-component GMM with varying dimension,
      using 3rd order moments. The reconstruction is performed from
      (a) 10K samples and (b) 200K samples.  Blue line with circles
      marks the Givens coordinate minimization method. Red line with
      triangles marks the tensor power method, and the black line is
      the optimal performance if all the GMM parameters are
      known. }\label{fig:tensor1}
    \end{small}
  \end{center}
\end{figure} 
\vskip -0.1in

We generated GMMs with the following parameters: number of dimensions
in $\{10,\, 20,\, 50,\, 100,\, 200\}$, number of samples sampled from
the model in $\{10K, 30K, 50K, 100K, 200K\}$. We used $20$ components,
each with a spherical variance of $2$. The centers were sampled from a
Gaussian distribution with an inverse-Wishart distributed covariance
matrix.  Given the samples, we then constructed the third order
moment, decomposed it, and reconstructed the model following the
procedure outlined in \citet{anandkumar2012tensor}. We then clustered
the samples according to the reconstructed model, and measured the
{\em normalized mutual information} (NMI)
\citep{manning2008introduction} between the learned clustering and the
true clusters.

Figure \ref{fig:tensor1} compares the performance of the two methods
with the optimal NMI across dimensions. The coordinate minimization
method outperforms the tensor power method for the large sample size
(200K), whereas for small sample size (10K) the tensor power method
performs better for the intermediate dimensions. In Figure
\ref{fig:tensor2} we see the performance of both algorithms across all
sample sizes for dimension $= 100$. We see that the coordinate
minimization method again performs better for larger sample sizes. We
observed this phenomenon for 50 components as well, and for mixture
models with larger variance.

\begin{figure}
  \vskip -0.07in
  \begin{center}
    \centerline{\includegraphics[width=0.92\linewidth]{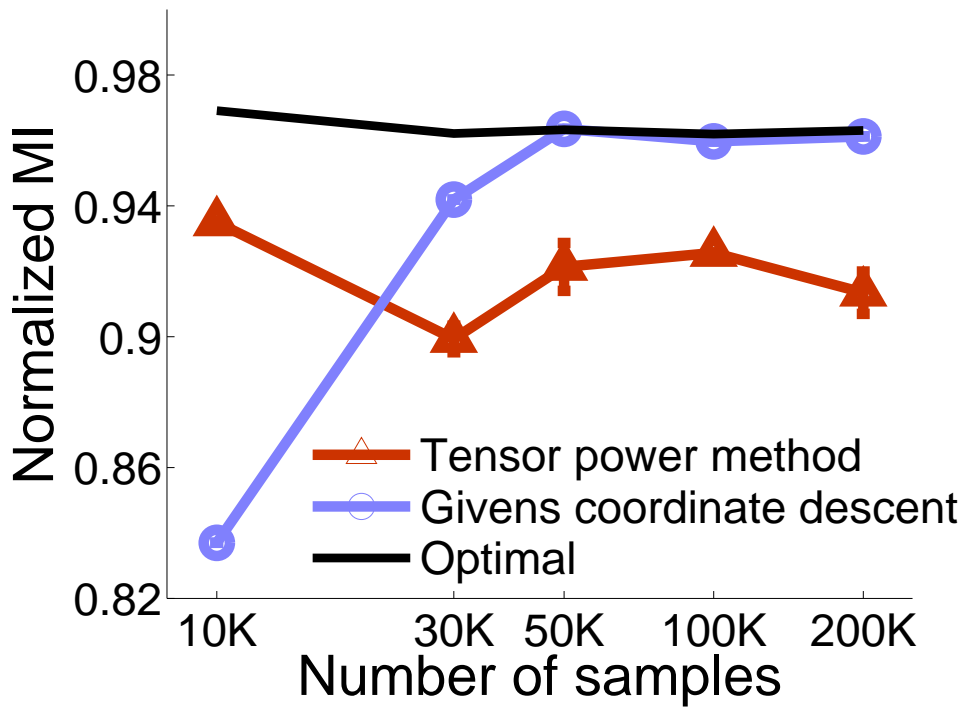}}
    \begin{small}
      \caption{Same task as Figure \ref{fig:tensor1}, but for fixed
      dimension $d = 100$ and varying number of
      samples. }\label{fig:tensor2}
    \end{small}
  \end{center}
  \vskip -0.01in
\end{figure} 
\vskip -0.05in


\section{Conclusion}
We described a framework to efficiently optimize 
differentiable functions over the manifold of orthogonal matrices.
The approach is based on Givens rotations, which we show can be 
viewed as the parallel of coordinate updates in Euclidean spaces. 
We prove the procedure's convergence to a local optimum.

Using this framework, we developed algorithms for two  
unsupervised learning problems. First, finding sparse principal 
components; and second, learning a Gaussian mixture model 
through orthogonal tensor decomposition.

We expect that the framework can be further extended to other problems 
requiring learning over orthogonal matrices including ICA. 
Moreover, coordinate descent approaches have some inherent advantages 
and are sometimes better amenable to parallelization. 
Developing distributed Givens-rotation algorithms would be an 
interesting future research direction.



\bibliographystyle{icml2014}

\bibliography{givens_arxiv1}

\appendix
\newenvironment{appTheorem}[2][Theorem]{\begin{trivlist}
\item[\hskip \labelsep {\bfseries #1}\hskip \labelsep {\bfseries #2}]}{\end{trivlist}}

\newenvironment{appDefinition}[2][Definition]{\begin{trivlist}
\item[\hskip \labelsep {\bfseries #1}\hskip \labelsep {\bfseries #2}]}{\end{trivlist}}

\section{Proofs of theorems of section 3}

Below we use a slightly modified definition of Algorithm \ref{alg1}. The difference lies only in the sampling procedure, and
is essentially a technical difference to ensure that each coordinate step indeed improves the objective or lies at an optimum, so that the proofs could be stated more succinctly.

\begin{algorithm}
\begin{algorithmic}
\REQUIRE Differentiable objective function $f$, initial matrix $U_0 \in \Od$
\STATE $t=0$
\WHILE{not converged}
\STATE 1. Sample coordinate pairs $(i(t),j(t))$ such that $1 \leq i(t) < j(t) \leq d$ uniformly at random without replacement, until the objective function can improve
\STATE 2. $U_{t+1} = \underset{\theta }{\operatorname{argmin}} \; f\left(U_t \cdot G(i,j,\theta)\right)$.
\STATE 3. $t = t+1$.
\ENDWHILE
\end{algorithmic}
\caption{Riemannian coordinate minimization on $\Od$, sampling variant }
\label{alg1APP}
\end{algorithm}

\begin{definition}
A point $U_{*} \in \Od$ is \emph{asymptotically stable} with respect
to Algorithm \ref{alg1APP} if it has a neighborhood $\mathcal{V}$ of $U_{*}$ such that
all sequences generated by Algorithm \ref{alg1APP} with starting point
$U_0 \in \mathcal{V}$ converge to $U_{*}$.
\end{definition}

\begin{appTheorem}{1.}{Convergence to local optimum} \\
(1) The sequence of iterates $U_t$ of Algorithm \ref{alg1APP} satisfies:
  $\lim_{t \to \infty} ||\nabla f (U_t) || = 0$. This means that the
  accumulation points of the sequence $\{U_t\}_{t=1}^{\infty} $ are
  critical points of $f$.  \\
(2) Assume the critical points of $f$ are isolated. Let $U_{*}$ be a
critical point of $f$. Then $U_{*}$ is a local minimum of $f$ if and
only if it is asymptotically stable with regard to the sequence
generated by Algorithm \ref{alg1APP}.
\end{appTheorem}

\begin{proof}
(1)
Algorithm \ref{alg1APP} is obtained by taking a step in each iteration $t$ in the direction of the tangent vector $Z_t$, such that for the coordinates $(i(t),j(t))$ we have $(Z_t)_{ij} = -(\nabla f(U_t))_{ij}$, $(Z_t)_{ji} = -(\nabla f(U_t))_{ji}$ , and $(Z_t)_{kl} = 0$ for all other coordinates $(k,l)$. 

  The sequence of tangent vectors $Z_t \in T_{U_t} \Od$ is easily seen to be gradient related: $\limsup{k\rightarrow \infty} \langle \nabla f (U_t), Z_t \rangle < 0$ \footnote{To obtain a rigorous proof we slightly complicated the sampling procedure in line 1 of Algorithm \ref{alg1}, such that coordinates with 0 gradient are not resampled until a non-zero gradient is sampled.}. This follows from $Z_t$ being equal to exactly two coordinates of $\nabla f(U_t)$, with all other coordinates being 0.

Using the optimal step size as we do assures at least as large an increase $f(U_t) - f(U_{t+1})$ as using the Armijo step size rule \citep{armijo1966minimization,bertsekas1999nonlinear}. Using the fact that the manifold $\Od$ is compact, we obtain by theorem 4.3.1 and corrolary 4.3.2 of \citet{absil2009optimization} that $\lim_{t \to \infty} ||\nabla f (U_t) ||  = 0$

(2) Since Algorithm \ref{alg1APP} produces a monotonically decreasing sequence $f(U_t)$, and since the manifold $\Od$ is compact, we are in the conditions of Theorems 4.4.1 and 4.4.2 of \citet{absil2009optimization}. These imply that the only critical points which are local minima are asymptotically stable.

\end{proof}

We now provide a rate of convergence proof. This proof is a Riemannian version of the proof for the rate of convergence of Euclidean random coordinate descent for non-convex functions given by \citet{patrascu2013efficient}. 

\begin{definition}
For an iterate $t$ of Algorithm \ref{alg1APP}, and a set of indices $(i(t),j(t))$, we define the auxiliary single variable function $g_t^{ij}$ :
\begin{equation}
\label{eq:defgtij}
g_t^{ij}(\theta) =  f\left(U_t \cdot 	G(i,j,\theta)\right),
\end{equation}
\end{definition}
Note that $g_t^{ij}$ are differentiable and periodic with a period of $2 \pi$. Since $\Od$ is compact and $f$ is differentiable there exists a single Lipschitz constant $L(f) > 0$ for all $g_t^{ij}$.

\begin{appTheorem}{2.}{Rate of convergence}\label{appth:rate} \\
Let $f$ be a continuous function with $L$-Lipschitz directional derivatives \footnote{Because $\Od$ is compact, any function $f$ with a continuous second-derivative will obey this condition.}. Let $U_t$ be the sequence generated by Algorithm \ref{alg1APP}. 
For the sequence of Riemannian gradients $\nabla f(U_t) \in T_{U_t} \Od$ we have: 
\begin{equation}
\underset{0 \leq t \leq T}{\operatorname{max}} E \left[ ||\nabla f(U_t)||_2^2 \right] \leq \frac{L\cdot d^2\left( f(U_0) -f_{min} \right)}{T+1} \quad .
\end{equation}
\end{appTheorem}
 

\begin{lemma}
\label{lm:boundedAPP}
Let $g: \reals \rightarrow \reals$ be a periodic differentiable function, with period $2 \pi$, and $L-$Lipschitz derivative $g'$. Then there for all $\theta \in [-\pi \; \pi]$:
$g(\theta) \leq g(0) + \theta g'(0) + \frac{L}{2} \theta^2$.
\end{lemma}	
\begin{proof}
We have for all $\theta$, \\
$ |g'(\theta) - g'(0)| \leq L |\theta|$.
We now have:
$g(\theta) - g(0) - \theta g'(0) = \int_0^{\theta}  g'(\tau) - g'(0) d\tau \leq \int_0^{\theta}  |g'(\tau) - g'(0)| d\tau  \leq \int_0^{\theta}  L |\tau| d\tau = \frac{L}{2} \theta ^2$.
\end{proof}

\begin{corollary}
\label{corr1}
Let $g = g_{i(t+1)j(t+1)}^{t+1}$. Under the conditions of Algorithm \ref{alg1APP}, we have: \\
$f(U_t) - f(U_{t+1}) \geq \frac{1}{2L} \nabla_{ij} f(U_t)^2$
for the same constant $L$ defined in \ref{lm:boundedAPP}.
\end{corollary}
\begin{proof}
By the definition of $g$ we have $f(U_{t+1}) = \underset{\theta }{\operatorname{min}} \; g(\theta)$, and we also have $g(0) = f(U_t)$. Finally, by Eq. 1 of the main paper we have  $ \nabla_{ij} f(U_t) = g^{\prime}(0)$.
From Lemma \ref{lm:boundedAPP}, we have $  g(\theta) -g(0) \leq \theta g'(0) + \frac{L}{2} \theta^2$. Minimizing the right-hand side with respect to $\theta$, we see that $\underset{\theta }{\operatorname{min}} \; \{g(0) - g(\theta) \} \geq \frac{1}{2L} (g'(0))^2$. 
Substituting $f(U_{t+1})= \underset{\theta }{\operatorname{min}} \; g(\theta)$ ,$f(U_t) = g(0)$, and $ \frac{1}{2L} \nabla_{ij} f(U_t) = g^{\prime}(0)$ completes the result.
\end{proof}

\begin{proof}[Proof of Theorem \ref{appth:rate}]
By Corollary \ref{corr1}, we have $f(U_t) - f(U_{t+1}) \geq \frac{1}{2L} \nabla_{ij} f(U_t)^2$.
Recall that $\pm \nabla_{ij} f(U_t)$ is the $(i,j)$ and $(j,i)$ entry of  $\nabla f(U_t)$.
If we take the expectation of both sides with respect to a uniform random choice of indices $i,j$ such that $1 \leq i < j \leq d$, we have:
\begin{equation}
\label{eq:ineq}
E\left[f(U_{t}) - f(U_{t+1})\right] \geq \frac{1}{L\cdot d^2 )}  || \nabla f(U_t)||^2,
\end{equation}

Summing the left-hand side gives a telescopic sum which can be bounded by $f(U_0) - \underset{U \in \Od }{\operatorname{min}} f(U) = f(U_0) - f_{min}$.
Summing the right-hand side and using this bound, we obtain
\begin{equation}
\sum_{t=0}^T E \left[ ||\nabla f(U_t)||_2^2 \right] \leq L \cdot d^2 (f(U_0) -f_{min} )
\end{equation}
This means that $\underset{0 \leq t \leq T }{\operatorname{min}} E \left[ ||\nabla f(U_t)||_2^2 \right] \leq \frac{L \cdot d^2  (f(U_0) -f_{min} )}{T+1}$. \qedhere

\end{proof}

\section{Proofs of theorems of section 5}

\begin{appDefinition}{4.}
\label{appdef:orth_dec}
A tensor $T$ is \textit{orthogonally decomposable} if there exists an
orthonormal set of vectors $v_1, \ldots v_d \in \reals^d$, and positive scalars $\lambda_1,
\ldots \lambda_d >0$ such that:
\begin{equation}
\label{appeq:orth_dec}
T = \sum_{i=1}^d \lambda_i (v_i \otimes v_i \otimes v_i).
\end{equation}
\end{appDefinition}

\begin{appTheorem}{3.}
\label{appth:unq}
Let $T \in R^{d \times d \times d}$ have an orthogonal decomposition as in Definition \ref{def:orth_dec}, and consider the optimization problem
\begin{equation}
\label{appeq:obj}
\underset{U \in \Od}{\operatorname{max}} f(U) = \sum_{i=1}^d T(u_i,u_i,u_i),
\end{equation}
where $U = [u_1 \, u_2 \, \ldots \, u_d]$.
The stable stationary points of the problem are exactly orthogonal matrices $U$ such that $u_i = v_{\pi(i)}$ for a permutation $\pi$ on $[d]$. The maximum value they attain is $\sum_{i=1}^d \lambda_i$.
\end{appTheorem}

\begin{proof}
For a tensor $T'$ denote $\textrm{vec}(T') \in \reals^{d^3}$ the vectorization of $T'$ using some fixed order of indices.
Set $\hat{T}(U) = \sum_{i=1}^d (u_i \otimes u_i \otimes u_i)$, with $\hat{T}(U)_{abc} = \sum_{i=1}^d u_{ia} u_{ib} u_{ic}$.
The sum of trilinear forms in Eq. \ref{appeq:obj} is equivalent to the inner product in $\reals^{d^3}$ between $\hat{T}(U)$ and $T$:
$\sum_{i=1}^d T(u_i,u_i,u_i) = \sum_{i=1}^d \sum_{abc} T_{abc} u_{ia} u_{ib} u_{ic} =  \sum_{abc} T_{abc} \left(\sum_{i=1}^d u_{ia} u_{ib} u_{ic} \right) = \sum_{abc} T_{abc} \hat{T}(U)_{abc} = 
\textrm{vec}(T) \bm{\cdot} \textrm{vec}(\hat{T}(U))$. 
Consider the following two facts:\\
(1) \hskip 0.05in $\hat{T}(U)_{abc} \leq 1 \; \forall a,b,c = 1 \ldots d$: since the vectors $u_i$ are orthogonal, all their components $u_{ia} \leq 1$. Thus $\hat{T}(U)_{abc} = \sum_{i=1}^d u_{ia} u_{ib} u_{ic} \leq \sum_{i=1}^d u_{ia} u_{ib} = \leq 1$, where the last inequality is because the sum is the inner product of two rows of an orthogonal matrix.\\
(2) \hskip 0.05in $||\textrm{vec}(\hat{T}(U))||_2^2 =  d$. This is easily checked by forming out the sum of squares explicitly, using the orthonormality of the rows and columns of the matrix $U$. \\
Assume without loss of generality that $V = I_d$. This is because we may replace the terms $T(u_i,u_i,u_i)$ in the objective with $T(V^Tu_i,V^Tu_i,V^Tu_i)$, and because the manifold
$V^T \Od$ is identical to $\Od$. Thus we have that $T$ is a diagonal tensor, with $T_{aaa} = \lambda_a > 0$, $a = 1 \ldots d$. Considering facts (1) and (2) above, we have the following inequality:
\begin{align}
& \underset{U \in \Od}{\operatorname{max}} \sum_{i=1}^d T(u_i,u_i,u_i) =  \underset{U \in \Od}{\operatorname{max}} \textrm{vec}(\hat{T}(U)) \bm{\cdot} T  \leq  \label{eq:innerp1} \\
& \underset{\hat{T}}{\operatorname{max}} \; \textrm{vec}(\hat{T}) \bm{\cdot} T  
\quad  s.t. \quad  || \textrm{vec}(\hat{T}) ||_{\infty} \leq 1  \, \wedge \, || \textrm{vec}(\hat{T}) ||_2^2 = d . \label{eq:innerp2} 
\end{align}

$T$ is diagonal by assumption, with exactly $d$ non-zero entires. Thus the maximum of (\ref{appeq:orth_dec}) is attained if and only if $\hat{T}_{aaa} = 1$, $a=1 \ldots d$, and all other entries of $\hat{T}$ are $0$. The value at the maximum is then $\sum_{i=1}^d \lambda_i$.

The diagonal ones tensor $\hat{T}$ can be decomposed into $\sum_{i=1}^d e_i \otimes e_i \otimes e_i$. Interestingly, in the tensor case, unlike in the matrix case, the decomposition of orthogonal tensors is \textit{unique} upto permutation of the factors \citep{Kruskal197795,kolda2009tensor}. Thus, the only solutions which attain the maximum of \ref{eq:innerp1} are those where $u_i = e_{\pi(i)}$, $i=1 , \ldots d$. 
\end{proof}

\section{Algorithm for streaming sparse PCA}
Following are the details for the streaming sparse PCA version of our algorithm used in the experiments of section 4.
The algorithm starts with running the original coordinate minimization procedure on the first $m$ samples. It then 
chooses the column with the least $l_2$ and replaces it with a new data sample, and then reoptimizes on the new set of samples.
There is no need for it to converge in the inner iterations, and in practice we found that order $m$ steps after each new sample 
are enough for good results.
\begin{algorithm}[t]
\begin{algorithmic}
\REQUIRE Data stream $a_i \in \reals^d$,  number of sparse principal components $m$, initial matrix $U_0 \in \Om$,
sparsity parameter $\gamma \geq 0$, number of inner iterations $L$.

\STATE $AU = [a_1 a_2 \ldots a_m] \cdot U_0$ . //$AU$ is of size $d \times m$
\WHILE{not stopped}
\FOR{$t=1 \ldots L$}
\STATE 1. Sample uniformly at random a pair $(i(t),j(t))$ such that $1 \leq i(t) < j(t) \leq m$.
\STATE 2. $\theta_{t+1} = \underset{\theta }{\operatorname{argmax}}$ \\ $\sum_{k=1}^d ([|cos(\theta) (AU)_{ki(t)} + sin(\theta) (AU)_{kj(t)}| - \gamma]_{+}^2$ \\ $+  [|-sin(\theta) (AU)_{ki(t)} + cos(\theta) (AU)_{kj(t)}| - \gamma]_{+}^2 )$.
\STATE 3.$AU = AU \cdot G(i(t),j(t)),\theta_{t+1})$.
\ENDFOR
\STATE 4. $i_{min} = \underset{i=1 \dots m }{\operatorname{argmin}} ||(AU)_{:,i}||_2$. 
\STATE 5. Sample new data point $a_{new}$.
\STATE 6. $(AU)_{:,i_{min}} = a_{new}$.
\ENDWHILE
\STATE $Z = solveForZ(AU,\gamma)$  // Algorithm 6 of \\  \quad \citet{journee2010generalized}.
\ENSURE $Z \in \reals^{d \times m}$
\end{algorithmic}
\caption{Riemannian coordinate minimization for streaming sparse PCA }
\label{alg5}
\end{algorithm}

\end{document}